\documentclass{article}

\PassOptionsToPackage{numbers,sort&compress}{natbib}

\usepackage[preprint]{neurips_2026}

\usepackage[utf8]{inputenc} 
\usepackage[T1]{fontenc}    
\usepackage{hyperref}       
\usepackage{url}            
\usepackage{booktabs}       
\usepackage{amsfonts}       
\usepackage{amsmath}        
\usepackage{amssymb}        
\usepackage{amsthm}         
\newtheorem{theorem}{Theorem}
\usepackage{nicefrac}       
\usepackage{microtype}      
\usepackage{xcolor}         
\usepackage{graphicx}       
\usepackage{subcaption}     
\usepackage{tikz}           
\usepackage{float}          
\usepackage{array}          
\usetikzlibrary{positioning,arrows.meta,shapes,decorations.pathreplacing,calc}

\newcommand{\indicator}{\mathbf{1}}

\makeatletter
\renewenvironment{table}
  {\setlength{\abovecaptionskip}{8\p@}%
   \setlength{\belowcaptionskip}{\z@}%
   \@float{table}}
  {\end@float}
\makeatother

\definecolor{berkeleyblue}{RGB}{0,50,98}
\definecolor{berkeleylightblue}{RGB}{59,126,161}
\definecolor{berkeleygold}{RGB}{253,181,21}
\definecolor{berkeleydarkgold}{RGB}{196,130,14}

\title{KLip-PPO: A per-sample KL perspective on PPO-Clip}

\author{%
  Riccardo Colletti\thanks{Equal contribution.} \\
  University of California, Berkeley \\
  \texttt{riccardo\_colletti [at] berkeley.edu} \\
  \And
  Robin Holzinger\footnotemark[1] \\
  University of California, Berkeley \\
  \texttt{robin.holzinger [at] berkeley.edu}
}

\begin{document}

\maketitle
\begingroup
\renewcommand\thefootnote{}
\footnotetext{\begin{tabular}[t]{@{}l@{}}Code: \url{https://github.com/learning-mechanisms/KLip-PPO}\\ Project page: \url{https://klip-ppo.org}\\ Public W\&B artifacts: \url{https://wandb.ai/KLip-PPO/KLip-PPO}\end{tabular}}
\endgroup

\begin{abstract}
  Proximal Policy Optimization (PPO) is the standard policy-gradient
  algorithm for on-policy reinforcement learning. The literature
  presents it in two forms, a clipped surrogate that bounds the
  importance ratio between successive policies and a
  Kullback--Leibler penalty between them. These forms are treated as
  separate algorithms with their own gradients, their own
  hyperparameters, and their own reference implementations, and a
  sizeable body of empirical work compares them. We show that the
  gradient of the clipped surrogate is reproduced exactly by a
  Kullback--Leibler surrogate whose coefficient varies per sample,
  with closed-form dependence on the importance ratio and the
  advantage. The identity holds at every minibatch step and across
  the entire inner loop, and on five MuJoCo continuous-control
  benchmarks the two losses produce indistinguishable training
  curves. The reformulation exposes a structural feature of
  the clipped surrogate that the $\min$ notation hides. PPO-Clip's
  implicit per-sample penalty is a step function at the boundary of
  the trust region, and the shape of this coefficient is the natural
  design axis for generalising the algorithm. We sketch the
  resulting follow-up directions in the discussion.
\end{abstract}

\section{Introduction}
\label{sec:intro}

Proximal Policy Optimization \citep{schulman2017ppo} is the default
policy-gradient algorithm for on-policy reinforcement learning. The
method approximates the trust-region step of TRPO
\citep{schulman2015trpo, kakade2002approximately} by maximising a
surrogate objective that keeps the new policy close to the rollout
policy. The original work proposes two surrogates. The first clips
the importance ratio between the new and old policies inside a fixed
band. The second adds an adaptive Kullback--Leibler penalty between
them. \citet{schulman2017ppo} report that the clipped variant
outperforms the penalty variant on MuJoCo continuous-control
benchmarks and recommend it as the default. The community has
followed the recommendation. PPO-Clip is the de facto choice in
modern open-source implementations \citep{raffin2021sb3, huang2022cleanrl},
and the clipped surrogate is the building block of token-level
extensions such as GRPO \citep{shao2024deepseekmath} that underpin
recent reasoning models \citep{deepseekai2025r1}.

Subsequent empirical work has scrutinised PPO from many angles.
\citet{engstrom2020implementation} show that ``code-level''
optimisations explain most of PPO's gain over TRPO and that the clip
mechanism itself is not load-bearing for performance.
\citet{ilyas2020closer} demonstrate that auxiliary optimisations,
rather than the clip term, are what actually maintain the trust
region. \citet{andrychowicz2021matters} train over $2.5 \times 10^5$
agents to compare design choices and recommend PPO-Clip among five
policy losses, though they do not run a standalone PPO-KL variant.
\citet{hsu2020revisiting} report that KL-regularised PPO matches or
outperforms the clipped variant outside MuJoCo with Gaussian
policies, and \citet{sun2022espo} argue that ratio clipping is not
necessary in PPO at all. Across this body of work the clip and KL
forms are treated as alternative algorithmic choices to be compared
empirically.

We show that this treatment misunderstands the relationship between
the two surrogates. The per-sample gradient of PPO-Clip is reproduced
exactly by a Kullback--Leibler surrogate whose coefficient varies
per sample, with closed-form dependence on the importance ratio and
the advantage. The identity holds at every minibatch step and across
the entire inner loop. On HalfCheetah, Hopper, Walker2d, Ant, and
Humanoid \citep{todorov2012mujoco} the two losses produce
indistinguishable training curves. Where the original PPO
paper notes that $L^{\mathrm{CLIP}}$ and the unclipped surrogate
agree to first order around $\theta_{\mathrm{old}}$
\citep{schulman2017ppo}, we strengthen the statement to a per-sample
identity that holds at every parameter configuration. Recent work
has analysed gradient-level relationships between different KL
formulations \citep{yao2025rethinkingkl} and proposed unified
clip-plus-KL design frameworks \citep{zhang2026rpg}, but to our
knowledge no prior work establishes the per-sample equivalence
between PPO-Clip and PPO-KL itself.

Making the implicit coefficient explicit clarifies the position of
PPO-Clip in the policy-optimisation landscape. The clipped surrogate
is a per-sample KL penalty whose coefficient is a step function on
the trust-region boundary, and the \textbf{shape of this
coefficient} is the natural design axis for generalising the
algorithm. Soft relaxations of the step, asymmetric and
position-conditioned penalties, and off-policy extensions all fit
inside the same template; we sketch these directions in
Section~\ref{sec:future}.

The paper is organised as follows.
Section~\ref{sec:background} reviews PPO-Clip, PPO-KL, and the
existing comparisons between them.
Section~\ref{sec:theorem} states and proves the per-sample gradient
identity. Section~\ref{sec:experiments} validates the identity
empirically on five MuJoCo continuous-control benchmarks.
Section~\ref{sec:future} surveys natural extensions of the framework.
Section~\ref{sec:conclusion} concludes.

\section{Background}
\label{sec:background}

We adopt the standard on-policy actor-critic setting
\citep{sutton2018reinforcement, sutton2000policygradient}. A behaviour
policy $\pi_\theta$ collects a rollout of $N$ trajectories of horizon
$H$ in a Markov decision process. For each sample $(i, t)$ the
rollout records a state $s_t^{(i)}$, an action $a_t^{(i)}$, a reward
$r_t^{(i)}$, and an advantage estimate $\hat{A}_t^{(i)}$ produced by
generalized advantage estimation \citep{schulman2016gae}. An update
step lifts the rollout policy $\pi_\theta$ to a new policy
$\pi_{\theta'}$ by repeated stochastic gradient steps on a surrogate
objective \citep{kakade2002approximately, schulman2015trpo}; the
importance ratio $w_t^{(i)} = \pi_{\theta'}(a_t^{(i)} \mid s_t^{(i)})
/ \pi_\theta(a_t^{(i)} \mid s_t^{(i)})$ tracks how much $\pi_{\theta'}$
has moved away from $\pi_\theta$ on the sampled actions. The two
surrogate objectives we study are the proximal-policy
\citep{schulman2017ppo} pair: PPO-Clip and PPO-KL.

\subsection{The surrogate objective}

Let $\pi_\theta$ denote the policy parameterised by $\theta$. PPO
collects a rollout under the behaviour policy $\pi_\theta$ and
updates the parameters to $\theta'$ by repeated stochastic gradient
steps on a surrogate objective. The starting point is the importance
sampled return
\citep{kakade2002approximately, schulman2015trpo}, written for a
batch of $N$ episodes of horizon $H$ as
\[
L_{\mathrm{IS}}(\theta')
\;=\;
\frac{1}{N}\sum_{i=1}^{N}\sum_{t=1}^{H}
w_t^{(i)}\,\hat{A}_t^{(i)},
\qquad
w_t^{(i)} \;=\;
\frac{\pi_{\theta'}\!\left(a_t^{(i)} \mid s_t^{(i)}\right)}
     {\pi_{\theta}\!\left(a_t^{(i)} \mid s_t^{(i)}\right)},
\]
where $\hat{A}_t^{(i)}$ is an estimator of the advantage at the
transition $(s_t^{(i)}, a_t^{(i)})$. Direct optimisation of
$L_{\mathrm{IS}}$ is unstable because the importance ratios
$w_t^{(i)}$ can become large when $\theta'$ drifts away from
$\theta$. \citet{schulman2015trpo} address the instability by
constraining the KL divergence between $\pi_{\theta'}$ and
$\pi_\theta$. \citet{schulman2017ppo} replace the constrained
optimisation by one of two penalised first-order surrogates.

\subsection{PPO-Clip}

The clipped surrogate of \citet{schulman2017ppo} is
\begin{equation}
L_{\mathrm{CLIP}}(\theta')
\;=\;
\frac{1}{N}\sum_{i=1}^{N}\sum_{t=1}^{H}
\min\!\Bigl\{
  w_t^{(i)}\,\hat{A}_t^{(i)},\;\;
  \mathrm{clip}\!\left(w_t^{(i)},\,1-\epsilon,\,1+\epsilon\right)\,
  \hat{A}_t^{(i)}
\Bigr\},
\label{eq:clip-obj}
\end{equation}
with $\epsilon \in (0,1)$ a fixed hyperparameter (typically
$\epsilon = 0.2$). The $\min$ acts as a one-sided trust region. When
$\hat{A}_t^{(i)} > 0$, increasing $w_t^{(i)}$ beyond $1+\epsilon$ no
longer increases the loss; when $\hat{A}_t^{(i)} < 0$, decreasing
$w_t^{(i)}$ below $1-\epsilon$ no longer increases the loss. This
prevents the optimiser from exploiting large positive surrogate
values that would correspond to large policy changes.

\subsection{PPO-KL}

The KL-penalised surrogate of \citet{schulman2017ppo} adds an
explicit divergence term,
\[
\mathcal{L}_{\mathrm{KL}}(\theta')
\;=\;
L_{\mathrm{IS}}(\theta')
\;-\;
\beta\,\widehat{D}_{\mathrm{KL}}\!\bigl[\pi_\theta \,\Vert\, \pi_{\theta'}\bigr],
\]
where $\widehat{D}_{\mathrm{KL}}$ is an empirical KL estimate and
$\beta > 0$ is a scalar coefficient that is held fixed or adapted at
the end of each outer update to keep the empirical KL near a target
\citep{schulman2017ppo, heess2017dppo}. In the fixed-$\beta$ form,
$\beta$ is shared across samples and across training. In the
adaptive form, $\beta$ is multiplied or divided by a constant
whenever the empirical KL exceeds or falls below the target.

In both forms, $\beta$ is a single scalar applied uniformly to all
samples in the batch. The comparison of clipping against scalar-KL
penalisation has been studied extensively
\citep{schulman2017ppo, engstrom2020implementation, ilyas2020closer,
andrychowicz2021matters, hsu2020revisiting, sun2022espo}. The two
surrogates are taken throughout this literature as algorithmically
distinct; the gradient identity we establish in
Section~\ref{sec:theorem} shows that, at the per-sample level, they
are not.

\section{The per-sample KL view of PPO-Clip}
\label{sec:theorem}

\subsection{Partition of the rollout}

The gradient of $L_{\mathrm{CLIP}}$ depends on which of the two
arguments of the inner $\min$ is active for each sample. Fix
$\theta'$ and define the importance ratio $w_t^{(i)}$ as in
Section~\ref{sec:background}. The pairs $(i,t)$ partition into three
disjoint index sets,
\begin{align*}
\mathcal{I}_{\mathrm{in}}
  &\;=\;
  \bigl\{(i,t) \;:\; w_t^{(i)} \in [1-\epsilon,\, 1+\epsilon]\bigr\}
  \;\cup\;
  \bigl\{(i,t) \;:\; \hat{A}_t^{(i)} = 0\bigr\}, \\
\mathcal{I}_{\mathrm{kill}}
  &\;=\;
  \bigl\{(i,t) \;:\; w_t^{(i)} > 1+\epsilon \text{ and } \hat{A}_t^{(i)} > 0\bigr\}
  \;\cup\;
  \bigl\{(i,t) \;:\; w_t^{(i)} < 1-\epsilon \text{ and } \hat{A}_t^{(i)} < 0\bigr\}, \\
\mathcal{I}_{\mathrm{pass}}
  &\;=\;
  \bigl\{(i,t) \;:\; w_t^{(i)} > 1+\epsilon \text{ and } \hat{A}_t^{(i)} < 0\bigr\}
  \;\cup\;
  \bigl\{(i,t) \;:\; w_t^{(i)} < 1-\epsilon \text{ and } \hat{A}_t^{(i)} > 0\bigr\}.
\end{align*}
Intuitively, $\mathcal{I}_{\mathrm{in}}$ contains the samples for
which the clip is inactive (together with the zero-advantage samples,
which contribute no gradient to either surrogate), $\mathcal{I}_{\mathrm{kill}}$ the samples
for which the clip suppresses the gradient because the policy is
already moving in the advantage-improving direction, and
$\mathcal{I}_{\mathrm{pass}}$ the samples for which the clip leaves
the unclipped term active because the policy is moving against the
advantage.

\subsection{Gradient of PPO-Clip}

The gradient of a sum is the sum of gradients, and the gradient of
the inner $\min$ on each sample depends on which of its two arguments
is active.

\paragraph{Case $\mathcal{I}_{\mathrm{in}}$:} when $w_t^{(i)} \in
[1-\epsilon, 1+\epsilon]$, the clipping operator leaves $w_t^{(i)}$
unchanged, so $\mathrm{clip}(w_t^{(i)}, 1-\epsilon, 1+\epsilon) =
w_t^{(i)}$ and both arguments of the $\min$ coincide:
\[
\min\!\bigl\{ w_t^{(i)}\hat{A}_t^{(i)},\;\;
              \mathrm{clip}(w_t^{(i)},\,1-\epsilon,\,1+\epsilon)\hat{A}_t^{(i)} \bigr\}
\;=\; w_t^{(i)}\hat{A}_t^{(i)}.
\]
Its gradient is $\hat{A}_t^{(i)}\,\nabla_{\theta'} w_t^{(i)}$.

\paragraph{Case $\mathcal{I}_{\mathrm{kill}}$:} consider first the
subcase $w_t^{(i)} > 1+\epsilon$ and $\hat{A}_t^{(i)} > 0$. The
clipped value $\mathrm{clip}(w_t^{(i)}, 1-\epsilon, 1+\epsilon)$
saturates at $1+\epsilon$, so the clipped term equals $(1+\epsilon)
\hat{A}_t^{(i)}$, while the unclipped term equals
$w_t^{(i)}\hat{A}_t^{(i)}$. Because $\hat{A}_t^{(i)} > 0$ and
$w_t^{(i)} > 1+\epsilon$, the clipped term is the smaller of the two
and the $\min$ selects it. The clipped value is constant in
$\theta'$, so its gradient vanishes. The symmetric subcase
$w_t^{(i)} < 1-\epsilon$ and $\hat{A}_t^{(i)} < 0$ is analogous: the
clipped term saturates at $(1-\epsilon)\hat{A}_t^{(i)}$, which is
smaller (more negative) than the unclipped term $w_t^{(i)}
\hat{A}_t^{(i)}$, and the gradient vanishes again. In both
subcases the per-sample gradient is zero.

\paragraph{Case $\mathcal{I}_{\mathrm{pass}}$:} consider the
subcase $w_t^{(i)} > 1+\epsilon$ and $\hat{A}_t^{(i)} < 0$. The
clipped term equals $(1+\epsilon)\hat{A}_t^{(i)}$ and the unclipped
term equals $w_t^{(i)}\hat{A}_t^{(i)}$. Because $\hat{A}_t^{(i)} <
0$ and $w_t^{(i)} > 1+\epsilon$, the unclipped term is more negative
and the $\min$ selects it. The symmetric subcase $w_t^{(i)} <
1-\epsilon$ and $\hat{A}_t^{(i)} > 0$ is analogous. In both
subcases the active term is the unclipped $w_t^{(i)}\hat{A}_t^{(i)}$
and its gradient is $\hat{A}_t^{(i)}\,\nabla_{\theta'} w_t^{(i)}$.

Combining the three cases,
\[
\nabla_{\theta'} L_{\mathrm{CLIP}}
\;=\;
\frac{1}{N} \sum_{(i,t)\,\in\,\mathcal{I}_{\mathrm{in}}\,\cup\,\mathcal{I}_{\mathrm{pass}}}
\hat{A}_t^{(i)}\, \nabla_{\theta'} w_t^{(i)},
\]
where samples in $\mathcal{I}_{\mathrm{kill}}$ contribute zero. To
make the dependence on the policy explicit, observe that
$\pi_{\theta}(a_t^{(i)} \mid s_t^{(i)})$ does not depend on
$\theta'$, so
\[
\nabla_{\theta'} w_t^{(i)}
\;=\;
\nabla_{\theta'} \!
\left[
\frac{\pi_{\theta'}(a_t^{(i)} \mid s_t^{(i)})}
     {\pi_{\theta}(a_t^{(i)} \mid s_t^{(i)})}
\right]
\;=\;
\frac{\nabla_{\theta'} \pi_{\theta'}\!\left(a_t^{(i)} \mid s_t^{(i)}\right)}
     {\pi_{\theta}\!\left(a_t^{(i)} \mid s_t^{(i)}\right)}.
\]
Substituting yields
\begin{equation}
\nabla_{\theta'} L_{\mathrm{CLIP}}
\;=\;
\frac{1}{N} \sum_{(i,t)\,\in\,\mathcal{I}_{\mathrm{in}}\,\cup\,\mathcal{I}_{\mathrm{pass}}}
\hat{A}_t^{(i)} \,
\frac{\nabla_{\theta'} \pi_{\theta'}\!\left(a_t^{(i)} \mid s_t^{(i)}\right)}
     {\pi_\theta\!\left(a_t^{(i)} \mid s_t^{(i)}\right)}.
\label{eq:clip-grad}
\end{equation}
Table~\ref{tab:index-sets} summarises the per-sample gradient
contribution on each index set.

\begin{table}[H]
\centering
\small
\renewcommand{\arraystretch}{1.6}
\resizebox{\textwidth}{!}{%
\begin{tabular}{@{} >{\centering\arraybackslash}p{2.4cm} p{4.8cm} >{\centering\arraybackslash}p{3.8cm} p{3.2cm} @{}}
\toprule
\textbf{Index set} & \textbf{Condition} & \textbf{Active term in $\min$} & \textbf{Gradient} \\
\midrule
$\mathcal{I}_{\mathrm{in}}$ & $w_t^{(i)} \in [1\!-\!\epsilon,\; 1\!+\!\epsilon]$\newline \textit{or}\newline $\hat{A}_t^{(i)} = 0$ & $w_t^{(i)}\hat{A}_t^{(i)}$ (both equal) & $\hat{A}_t^{(i)}\,\nabla_{\theta'} w_t^{(i)}$ \\
\midrule
$\mathcal{I}_{\mathrm{kill}}$ & $w_t^{(i)} > 1\!+\!\epsilon$, $\hat{A}_t^{(i)}\!>\!0$\newline \textit{or}\newline $w_t^{(i)} < 1\!-\!\epsilon$, $\hat{A}_t^{(i)}\!<\!0$ & clipped term (constant) & $\mathbf{0}$ {\footnotesize\color{red!70!black}(killed)} \\
\midrule
$\mathcal{I}_{\mathrm{pass}}$ & $w_t^{(i)} > 1\!+\!\epsilon$, $\hat{A}_t^{(i)}\!<\!0$\newline \textit{or}\newline $w_t^{(i)} < 1\!-\!\epsilon$, $\hat{A}_t^{(i)}\!>\!0$ & unclipped term $w_t^{(i)}\hat{A}_t^{(i)}$ & $\hat{A}_t^{(i)}\,\nabla_{\theta'} w_t^{(i)}$ \\
\bottomrule
\end{tabular}%
}
\caption{Per-sample contribution of the PPO-Clip objective on the three index sets. The gradient is killed only on $\mathcal{I}_{\mathrm{kill}}$, where further policy change would push the importance ratio further outside the trust region.}
\label{tab:index-sets}
\end{table}

\subsection{Gradient of PPO-KL}

The empirical KL estimate
$\widehat{D}_{\mathrm{KL}}\!\bigl[\pi_\theta \,\Vert\, \pi_{\theta'}\bigr]$
of Section~\ref{sec:background} can be written, up to a
$\theta'$-independent additive constant that drops out of the
gradient, as a sum of negative log-probabilities of the sampled
actions under the new policy. After absorbing the constant and
allowing the penalty coefficient to vary per sample, the surrogate
becomes
\[
\mathcal{L}_{\mathrm{KL}}(\theta')
\;=\;
\frac{1}{N}\sum_{i=1}^{N}\sum_{t=1}^{H}\!\Bigl[
  w_t^{(i)}\,\hat{A}_t^{(i)}
  +
  \beta_t^{(i)}\,\log \pi_{\theta'}\!\left(a_t^{(i)} \mid s_t^{(i)}\right)
\Bigr].
\]
The standard fixed and adaptive PPO-KL variants of
\citet{schulman2017ppo} correspond to the choice
$\beta_t^{(i)} \equiv \beta$, with $\beta$ either fixed or updated
between outer iterations.

For the importance-sampled return term,
\[
\nabla_{\theta'}\!\left[w_t^{(i)}\,\hat{A}_t^{(i)}\right]
\;=\;
\hat{A}_t^{(i)}\,\nabla_{\theta'} w_t^{(i)}
\;=\;
\hat{A}_t^{(i)}\,
\frac{\nabla_{\theta'} \pi_{\theta'}\!\left(a_t^{(i)} \mid s_t^{(i)}\right)}
     {\pi_\theta\!\left(a_t^{(i)} \mid s_t^{(i)}\right)},
\]
using the same calculation as in the PPO-Clip derivation. For the
penalty term, the coefficient $\beta_t^{(i)}$ is a \emph{stop-gradient}
(detached) coefficient: it is evaluated at the current $\theta'$ and
held fixed when differentiating, as is standard for the penalty
coefficient of a KL-penalised policy gradient and as the
implementation does. Its derivative in $\theta'$ vanishes by
convention, so
\[
\nabla_{\theta'}\!\left[\beta_t^{(i)}\log\pi_{\theta'}\!\left(a_t^{(i)} \mid s_t^{(i)}\right)\right]
\;=\;
\beta_t^{(i)}\,
\frac{\nabla_{\theta'} \pi_{\theta'}\!\left(a_t^{(i)} \mid s_t^{(i)}\right)}
     {\pi_{\theta'}\!\left(a_t^{(i)} \mid s_t^{(i)}\right)}.
\]
Multiplying numerator and denominator of the right-hand side by
$\pi_\theta(a_t^{(i)} \mid s_t^{(i)})$ converts the rollout-policy
denominator to the same form as the importance-sampled term,
\[
\beta_t^{(i)}\,
\frac{\nabla_{\theta'} \pi_{\theta'}\!\left(a_t^{(i)} \mid s_t^{(i)}\right)}
     {\pi_{\theta'}\!\left(a_t^{(i)} \mid s_t^{(i)}\right)}
\;=\;
\beta_t^{(i)}\,
\frac{\nabla_{\theta'} \pi_{\theta'}\!\left(a_t^{(i)} \mid s_t^{(i)}\right)}
     {\pi_\theta\!\left(a_t^{(i)} \mid s_t^{(i)}\right)}\,
\cdot\,
\frac{\pi_\theta\!\left(a_t^{(i)} \mid s_t^{(i)}\right)}
     {\pi_{\theta'}\!\left(a_t^{(i)} \mid s_t^{(i)}\right)}
\;=\;
\frac{\beta_t^{(i)}}{w_t^{(i)}}\,
\frac{\nabla_{\theta'} \pi_{\theta'}\!\left(a_t^{(i)} \mid s_t^{(i)}\right)}
     {\pi_\theta\!\left(a_t^{(i)} \mid s_t^{(i)}\right)}.
\]
Summing the two contributions and factoring the common
$\nabla_{\theta'} \pi_{\theta'}\!\left(a_t^{(i)} \mid s_t^{(i)}\right)
/ \pi_\theta\!\left(a_t^{(i)} \mid s_t^{(i)}\right)$,
\begin{equation}
\nabla_{\theta'}\mathcal{L}_{\mathrm{KL}}
\;=\;
\frac{1}{N}\sum_{i=1}^{N}\sum_{t=1}^{H}
\frac{\nabla_{\theta'} \pi_{\theta'}\!\left(a_t^{(i)} \mid s_t^{(i)}\right)}
     {\pi_\theta\!\left(a_t^{(i)} \mid s_t^{(i)}\right)}\,
\Biggl[\,\hat{A}_t^{(i)} \;+\; \frac{\beta_t^{(i)}}{w_t^{(i)}}\,\Biggr].
\label{eq:kl-grad}
\end{equation}
Every sample contributes to the gradient. The penalty does not zero
out any term; it shifts the effective advantage of sample $(i,t)$
from $\hat{A}_t^{(i)}$ to
$\hat{A}_t^{(i)} + \beta_t^{(i)}/w_t^{(i)}$.

\subsection{Gradient identity}

Comparison of \eqref{eq:clip-grad} and \eqref{eq:kl-grad} yields the
main result.

\begin{theorem}[Per-sample gradient identity]
\label{thm:identity}
Let $L_{\mathrm{CLIP}}$ be the PPO-Clip surrogate of
\citet{schulman2017ppo} and let $\mathcal{L}_{\mathrm{KL}}$ be the
PPO-KL surrogate with per-sample stop-gradient coefficients
$\{\beta_t^{(i)}\}$, each evaluated at the current $\theta'$ and held
fixed under differentiation. Define
\begin{equation}
\beta_t^{(i)}
\;=\;
\begin{cases}
\;0,
  & (i,t) \in \mathcal{I}_{\mathrm{in}} \cup \mathcal{I}_{\mathrm{pass}}, \\[2pt]
\,-\,w_t^{(i)}\,\hat{A}_t^{(i)},
  & (i,t) \in \mathcal{I}_{\mathrm{kill}}.
\end{cases}
\label{eq:beta-def}
\end{equation}
Then
\[
\nabla_{\theta'} L_{\mathrm{CLIP}} \;=\; \nabla_{\theta'}\mathcal{L}_{\mathrm{KL}}
\]
at every $\theta'$ where $L_{\mathrm{CLIP}}$ is differentiable, namely
wherever no sample lies exactly on a clip boundary
$w_t^{(i)} = 1\pm\epsilon$; this excludes only a measure-zero set, on
which $L_{\mathrm{CLIP}}$ has a kink.
\end{theorem}

\begin{proof}
Fix a sample $(i,t)$ and write $g_t^{(i)}(\theta')$ for its
per-sample gradient contribution under either surrogate. From
\eqref{eq:clip-grad} and \eqref{eq:kl-grad},
\[
g_t^{(i)}\bigl(L_{\mathrm{CLIP}}\bigr)
\;=\;
\begin{cases}
\hat{A}_t^{(i)}\,
\dfrac{\nabla_{\theta'} \pi_{\theta'}\!\left(a_t^{(i)} \mid s_t^{(i)}\right)}
      {\pi_\theta\!\left(a_t^{(i)} \mid s_t^{(i)}\right)},
  & (i,t) \in \mathcal{I}_{\mathrm{in}} \cup \mathcal{I}_{\mathrm{pass}}, \\[12pt]
0,
  & (i,t) \in \mathcal{I}_{\mathrm{kill}},
\end{cases}
\]
and
\[
g_t^{(i)}\bigl(\mathcal{L}_{\mathrm{KL}}\bigr)
\;=\;
\frac{\nabla_{\theta'} \pi_{\theta'}\!\left(a_t^{(i)} \mid s_t^{(i)}\right)}
     {\pi_\theta\!\left(a_t^{(i)} \mid s_t^{(i)}\right)}\,
\Biggl[\,\hat{A}_t^{(i)} \;+\; \frac{\beta_t^{(i)}}{w_t^{(i)}}\,\Biggr].
\]
Under the choice \eqref{eq:beta-def}, the bracket of
$g_t^{(i)}(\mathcal{L}_{\mathrm{KL}})$ takes two values depending on
the index set.

\emph{Case $(i,t) \in \mathcal{I}_{\mathrm{in}} \cup
\mathcal{I}_{\mathrm{pass}}$.} The definition
\eqref{eq:beta-def} gives $\beta_t^{(i)} = 0$, so the bracket
reduces to $\hat{A}_t^{(i)}$. Hence
\[
g_t^{(i)}\bigl(\mathcal{L}_{\mathrm{KL}}\bigr)
\;=\;
\hat{A}_t^{(i)}\,
\frac{\nabla_{\theta'} \pi_{\theta'}\!\left(a_t^{(i)} \mid s_t^{(i)}\right)}
     {\pi_\theta\!\left(a_t^{(i)} \mid s_t^{(i)}\right)}
\;=\;
g_t^{(i)}\bigl(L_{\mathrm{CLIP}}\bigr).
\]

\emph{Case $(i,t) \in \mathcal{I}_{\mathrm{kill}}$.} The definition
\eqref{eq:beta-def} gives $\beta_t^{(i)} = -w_t^{(i)}\hat{A}_t^{(i)}$.
The bracket evaluates to
\[
\hat{A}_t^{(i)} \;+\; \frac{\beta_t^{(i)}}{w_t^{(i)}}
\;=\;
\hat{A}_t^{(i)} \;+\;
\frac{-w_t^{(i)}\hat{A}_t^{(i)}}{w_t^{(i)}}
\;=\;
\hat{A}_t^{(i)} - \hat{A}_t^{(i)}
\;=\; 0,
\]
where the cancellation uses $w_t^{(i)} > 0$ (since
$\pi_{\theta'}\!\left(a_t^{(i)} \mid s_t^{(i)}\right) > 0$ and
$\pi_\theta\!\left(a_t^{(i)} \mid s_t^{(i)}\right) > 0$ for any
sampled action). Therefore
\[
g_t^{(i)}\bigl(\mathcal{L}_{\mathrm{KL}}\bigr)
\;=\; 0
\;=\;
g_t^{(i)}\bigl(L_{\mathrm{CLIP}}\bigr).
\]

The per-sample contributions agree on every $(i,t)$, so the two
batch gradients agree:
\[
\nabla_{\theta'} L_{\mathrm{CLIP}}
\;=\;
\frac{1}{N}\sum_{i=1}^{N}\sum_{t=1}^{H} g_t^{(i)}\bigl(L_{\mathrm{CLIP}}\bigr)
\;=\;
\frac{1}{N}\sum_{i=1}^{N}\sum_{t=1}^{H} g_t^{(i)}\bigl(\mathcal{L}_{\mathrm{KL}}\bigr)
\;=\;
\nabla_{\theta'}\mathcal{L}_{\mathrm{KL}}.
\]
\end{proof}

\subsection{Interpretation}

Figure~\ref{fig:equivalence-table} reads the equivalence row by row:
for each combination of $w_t^{(i)}$ and $\hat{A}_t^{(i)}$, the
PPO-Clip gradient and the value of $\beta_t^{(i)}$ that reproduces it
under the KL surrogate are listed side by side.

\begin{figure}[H]
\centering
\begin{tikzpicture}[>=Stealth, scale=0.9, transform shape]

\node[font=\small\bfseries, text=berkeleyblue] at (0, 3.2) {Region};
\node[font=\small\bfseries, text=berkeleyblue] at (4.5, 3.2) {PPO-Clip gradient};
\node[font=\small\bfseries, text=berkeleyblue] at (9.5, 3.2) {Equivalent $\beta_t$};

\draw[berkeleyblue, line width=0.8pt] (-2.5, 2.8) -- (12.5, 2.8);

\node[font=\small, anchor=west] at (-2.4, 2.2) {$w_t^{(i)} \in [1\!-\!\epsilon,\, 1\!+\!\epsilon]$};
\node[font=\small] at (4.5, 2.2) {$\hat{A}_t^{(i)}\,\nabla w_t^{(i)}$};
\node[font=\small, text=green!50!black] at (9.5, 2.2) {$\beta_t = 0$};

\node[font=\small, anchor=west] at (-2.4, 1.4) {$w_t^{(i)} > 1\!+\!\epsilon$, $\hat{A}_t^{(i)} > 0$};
\node[font=\small, text=red!70!black] at (4.5, 1.4) {$0$ (killed)};
\node[font=\small] at (9.5, 1.4) {$\beta_t = -w_t^{(i)}\hat{A}_t^{(i)}$};

\node[font=\small, anchor=west] at (-2.4, 0.6) {$w_t^{(i)} > 1\!+\!\epsilon$, $\hat{A}_t^{(i)} < 0$};
\node[font=\small] at (4.5, 0.6) {$\hat{A}_t^{(i)}\,\nabla w_t^{(i)}$};
\node[font=\small, text=green!50!black] at (9.5, 0.6) {$\beta_t = 0$};

\node[font=\small, anchor=west] at (-2.4, -0.2) {$w_t^{(i)} < 1\!-\!\epsilon$, $\hat{A}_t^{(i)} > 0$};
\node[font=\small] at (4.5, -0.2) {$\hat{A}_t^{(i)}\,\nabla w_t^{(i)}$};
\node[font=\small, text=green!50!black] at (9.5, -0.2) {$\beta_t = 0$};

\node[font=\small, anchor=west] at (-2.4, -1.0) {$w_t^{(i)} < 1\!-\!\epsilon$, $\hat{A}_t^{(i)} < 0$};
\node[font=\small, text=red!70!black] at (4.5, -1.0) {$0$ (killed)};
\node[font=\small] at (9.5, -1.0) {$\beta_t = -w_t^{(i)}\hat{A}_t^{(i)}$};

\draw[berkeleyblue, line width=0.4pt] (-2.5, 1.8) -- (12.5, 1.8);
\draw[berkeleyblue, line width=0.4pt] (-2.5, 1.0) -- (12.5, 1.0);
\draw[berkeleyblue, line width=0.4pt] (-2.5, 0.2) -- (12.5, 0.2);
\draw[berkeleyblue, line width=0.4pt] (-2.5, -0.6) -- (12.5, -0.6);

\end{tikzpicture}
\caption{Summary of the per-sample equivalence. In the green rows the PPO-Clip gradient is reproduced by a PPO-KL surrogate with $\beta_t = 0$. In the red rows the PPO-Clip gradient is reproduced by a PPO-KL surrogate with $\beta_t = -w_t \hat{A}_t$, which is the value that kills the corresponding term of the per-sample gradient.}
\label{fig:equivalence-table}
\end{figure}

The coefficient \eqref{eq:beta-def} makes explicit what PPO-Clip
implicitly applies. The clip is a Kullback--Leibler penalty whose
strength is zero everywhere except on the killed region, where it
takes the value $-w_t^{(i)} \hat{A}_t^{(i)}$. The sign of the
coefficient is consistent with the trust-region intuition. When
$\hat{A}_t^{(i)} > 0$ and $w_t^{(i)} > 1+\epsilon$, the policy is
already over-weighting a beneficial action, and a negative
$\beta_t^{(i)}$ pulls $\log \pi_{\theta'}$ away from the current
direction; when $\hat{A}_t^{(i)} < 0$ and $w_t^{(i)} < 1-\epsilon$,
the policy is already under-weighting a harmful action, and a
positive $\beta_t^{(i)}$ stabilises it. In both cases the effective
advantage of the bracket in \eqref{eq:kl-grad} is driven exactly to
zero, reproducing the clip's behaviour.

The identity in Theorem~\ref{thm:identity} is stronger than the
first-order observation of \citet{schulman2017ppo} that
$L_{\mathrm{CLIP}}$ and the unclipped surrogate agree around
$w = 1$. It holds at every $\theta'$ off the clip boundary, every minibatch
step, and across the entire inner loop. Two closely related lines of recent
work have approached the clip and the KL term from related but
distinct angles: \citet{yao2025rethinkingkl} establish a gradient
equivalence between three KL estimators inside RLHF objectives, and
\citet{zhang2026rpg} propose a unified design framework for
KL-regularised policy gradient. Neither identifies the per-sample
coefficient \eqref{eq:beta-def} that turns PPO-Clip itself into a
KL penalty.

\section{Experiments}
\label{sec:experiments}

\subsection{Setup}
\label{sec:exp-setup}

We evaluate four objectives that share the surrogate of
Section~\ref{sec:background} and differ only in the policy loss. These
are PPO-Clip, fixed-$\beta$ PPO-KL, adaptive-$\beta$ PPO-KL, and the
per-sample PPO-KL of \eqref{eq:beta-def}; the first three are the
variants of \citet{schulman2017ppo} and the fourth is the construction
of Section~\ref{sec:theorem}. We train each on the five MuJoCo locomotion
tasks \citep{todorov2012mujoco} HalfCheetah-v4, Hopper-v4, Walker2d-v4,
Ant-v4, and Humanoid-v4 for $10^6$ environment steps over five seeds, and
include CartPole-v1 and LunarLander-v3 as a low-dimensional and a
discrete-action check.

All four variants share the trainer, the rollout collector, and the
value head, and use the standard PPO configuration of CleanRL
\citep{huang2022cleanrl} and Stable-Baselines3 \citep{raffin2021sb3}.
This configuration is a two-hidden-layer ($64$-$64$, $\tanh$)
actor-critic with orthogonal initialisation and a diagonal Gaussian
policy, GAE \citep{schulman2016gae} with $\gamma = 0.99$ and
$\lambda = 0.95$, Adam at $3 \cdot 10^{-4}$ with linear annealing,
gradient clipping at norm $0.5$, value-loss clipping at $0.2$,
observation and reward normalisation, and an inner loop of $K = 10$
epochs over size-$64$ minibatches of each $2048$-step rollout.

The variants differ in the trust-region knob. PPO-Clip uses
$\epsilon = 0.2$, fixed-$\beta$ PPO-KL uses $\beta = 1$, and
adaptive-$\beta$ PPO-KL follows \citet{schulman2017ppo}, updating $\beta$
once per rollout by a factor of two toward a target
$D_{\mathrm{KL}} = 0.02$, the second-order KL $\epsilon^2/2$ of the clip
radius. The fixed and adaptive variants penalise the analytic Gaussian
KL, while the per-sample variant penalises the sampled log-ratio
$-\log w_t^{(i)}$, the estimator for which Theorem~\ref{thm:identity}
holds. A knob sweep on CartPole-v1, HalfCheetah-v4, and Hopper-v4 over
$\epsilon \in \{0.1, 0.2, 0.3\}$, $\beta \in \{0.1, 0.3, 1, 3, 10\}$, and
$D_{\mathrm{KL}} \in \{0.003, 0.01, 0.02, 0.03, 0.1\}$ fixes each baseline
at its best value.

All public run histories and per-run reproducibility artifacts are available in
the Weights \& Biases project at \url{https://wandb.ai/KLip-PPO/KLip-PPO}.

\subsection{Results}
\label{sec:exp-results}

By Theorem~\ref{thm:identity}, PPO-Clip and the per-sample PPO-KL
surrogate share a gradient at every step. Their learning curves coincide
on all five MuJoCo tasks (Figure~\ref{fig:equivalence}), and their final
returns agree on every task (Table~\ref{tab:final-returns}). \citet{schulman2017ppo} showed
only that the clipped objective and the unclipped surrogate ($\beta = 0$)
agree to first order when the new policy is close to the one that collected
the rollout. The per-sample identity is
sharper, because with the coefficient $\beta_t$ of \eqref{eq:beta-def} in
place the equality becomes exact and holds at every $\theta'$, which is why
the curves stay together over the whole run, however far training moves the
policy.

The literature has consistently found clipping to outperform a KL
penalty on continuous control. The original PPO study reports this on
MuJoCo \citep{schulman2017ppo}, and subsequent benchmarks repeat it
\citep{engstrom2020implementation, andrychowicz2021matters}. Our
experiments reproduce the same ordering. PPO-KL with a fixed or an
adaptively tuned $\beta$ matches PPO-Clip on the easier tasks but falls
behind on the high-dimensional ones (Ant-v4 and Humanoid-v4), where the
policy must travel far from its initialisation and the trust region does
real work.

The per-sample identity explains the shortfall. Clipping constrains each
sample on its own terms, turning the penalty on only for the transitions
whose ratio has left the band and scaling it by that sample's ratio and
advantage. The scalar $\beta$ of PPO-KL, fixed or adaptive, instead
applies one value to every sample, so a $\beta$ large enough to restrain
the few runaway transitions
over-penalises the many well-behaved ones, and no single value reproduces
what the clip does pointwise \citep{hsu2020revisiting}. The per-sample coefficient of
\eqref{eq:beta-def} removes that limitation, setting the penalty separately
for each sample, and so reproduces the clip exactly.

\begin{figure}[tbp]
  \centering
  \begin{subfigure}[t]{0.32\textwidth}
    \includegraphics[width=\textwidth]{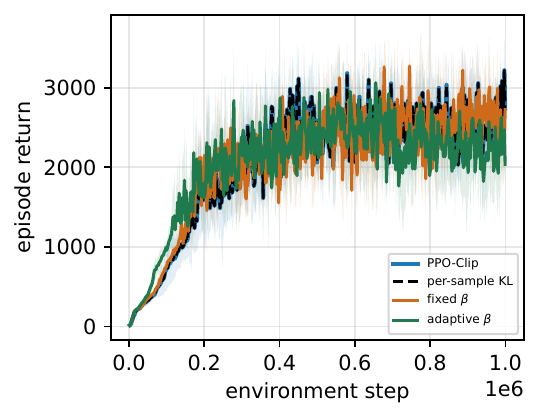}
    \caption{Hopper-v4}
  \end{subfigure}
  \hfill
  \begin{subfigure}[t]{0.32\textwidth}
    \includegraphics[width=\textwidth]{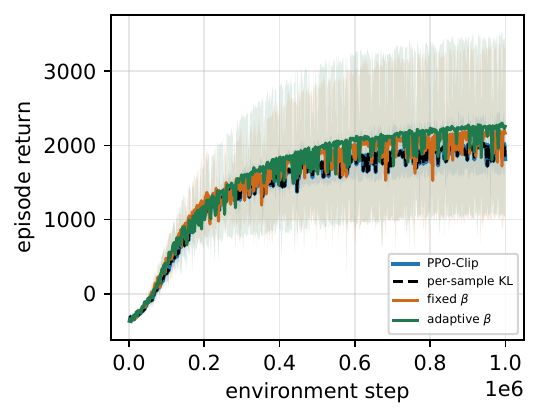}
    \caption{HalfCheetah-v4}
  \end{subfigure}
  \hfill
  \begin{subfigure}[t]{0.32\textwidth}
    \includegraphics[width=\textwidth]{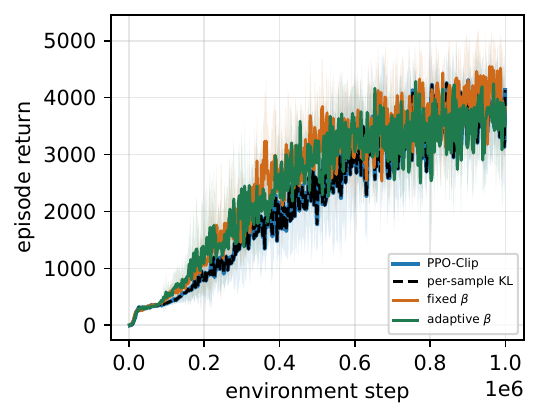}
    \caption{Walker2d-v4}
  \end{subfigure}

  \vspace{4pt}

  \begin{subfigure}[t]{0.32\textwidth}
    \includegraphics[width=\textwidth]{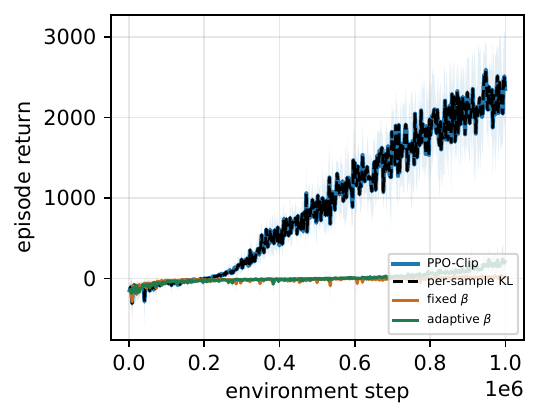}
    \caption{Ant-v4}
  \end{subfigure}
  \hspace{0.03\textwidth}
  \begin{subfigure}[t]{0.32\textwidth}
    \includegraphics[width=\textwidth]{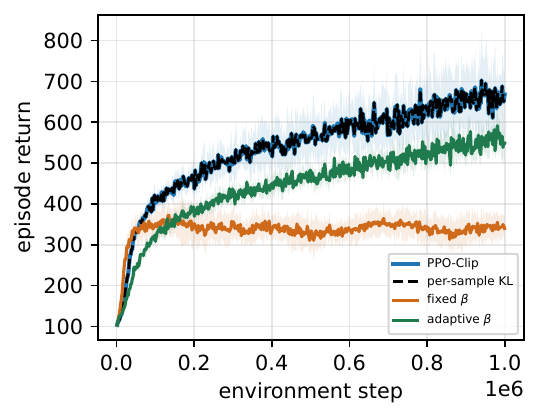}
    \caption{Humanoid-v4}
  \end{subfigure}
  \caption{Episode return on the five MuJoCo tasks (mean
  over $5$ seeds, $\pm$ std band). PPO-Clip and per-sample PPO-KL are
  indistinguishable on every task, while the fixed- and adaptive-$\beta$
  penalties fall behind on Ant-v4 (d) and Humanoid-v4 (e).}
  \label{fig:equivalence}
\end{figure}

\begin{table}[tbp]
  \centering
  \begin{tabular}{l rrrr}
\toprule
Task & PPO-Clip & Per-sample & Fixed $\beta$ & Adaptive $\beta$ \\
\midrule
Ant & $2193 \pm 361$ & $2193 \pm 361$ & $16 \pm 16$ & $159 \pm 86$ \\
CartPole & $478 \pm 18$ & $478 \pm 18$ & $484 \pm 23$ & $492 \pm 9$ \\
HalfCheetah & $1955 \pm 312$ & $1955 \pm 312$ & $2111 \pm 1066$ & $2207 \pm 1117$ \\
Hopper & $2598 \pm 278$ & $2598 \pm 278$ & $2616 \pm 208$ & $2236 \pm 365$ \\
Humanoid & $660 \pm 70$ & $660 \pm 70$ & $342 \pm 31$ & $553 \pm 22$ \\
LunarLander & $107 \pm 8$ & $107 \pm 8$ & $121 \pm 12$ & $125 \pm 20$ \\
Walker2d & $3717 \pm 425$ & $3717 \pm 425$ & $3995 \pm 402$ & $3639 \pm 374$ \\
\bottomrule
\end{tabular}

  \caption{Final return (mean $\pm$ std over $5$ seeds, last $10\%$ of
  training). PPO-Clip and per-sample PPO-KL are identical on every task; the
  scalar-$\beta$ variants fall behind on the high-dimensional tasks.}
  \label{tab:final-returns}
\end{table}

\section{Discussion}
\label{sec:conclusion}

Theorem~\ref{thm:identity} makes PPO-Clip's per-sample gradient exactly
the gradient of a Kullback--Leibler penalty,
\[
\nabla_{\theta'} L_{\mathrm{CLIP}} = \nabla_{\theta'}\,\mathbb{E}_t\!\bigl[\,w_t \hat{A}_t + \beta_t \log \pi_{\theta'}(a_t\mid s_t)\,\bigr],
\qquad
\beta_t = -\,w_t \hat{A}_t \,\indicator\!\left[(i,t)\in\mathcal{I}_{\mathrm{kill}}\right],
\]
whose coefficient is set for each sample by its own importance ratio and
advantage. The clip is in this sense a trust region acting in the space of
policy distributions, applied to the samples in
$\mathcal{I}_{\mathrm{kill}}$ that it would otherwise discard. We show in
Appendix~\ref{app:phi} that the same identity has a matching form in
importance-weight space, so the clipped objective~\eqref{eq:clip-obj}, the
weight-space penalty, and the per-sample KL penalty are \textbf{three
equivalent expressions of a single per-sample gradient}.

This reading has two consequences. First, it makes precise a question the
field has answered only empirically. Clipping and KL penalisation have been
compared through benchmark scores \citep{schulman2017ppo,
engstrom2020implementation, ilyas2020closer, andrychowicz2021matters,
hsu2020revisiting, sun2022espo}, with no theory relating the two
objectives. The identity reframes that comparison. Since clipping is itself a KL
penalty, the methods here differ only in how they set the penalty's
coefficient. PPO-KL uses one scalar $\beta$ for the whole batch, while
PPO-Clip, by Theorem~\ref{thm:identity}, uses the per-sample $\beta_t$ of
\eqref{eq:beta-def}. The benchmark gap read as clipping versus KL is thus a
gap between a scalar and a per-sample coefficient, which is where our
experiments locate it.
Second, it gives the penalty a \textbf{flexibility} the clipped form
hides. Once $\beta_t$ is written out, its step shape is one choice among
many, and replacing it with another, such as one that softens the boundary
or that varies with the individual sample, stays within the same surrogate
family and defines new algorithms (Section~\ref{sec:future}).

What governs the update is therefore the per-sample gradient coefficient
and not the surface choice between a clip and a KL term, a principle that
recent analyses of KL regularisation in language-model training
independently corroborate. \citet{yao2025rethinkingkl} find that a KL term
written as a loss and its associated per-sample reward coefficient induce
the same gradient, and \citet{zhang2026rpg} report that clipped and KL
objectives agree once their per-sample importance weights are aligned.
Neither isolates the exact coefficient $\beta_t = -w_t\hat{A}_t$ that makes
PPO-Clip itself a KL penalty, which is the identity established here; that
their gradient-level findings point the same way is evidence that the
per-sample view this identity rests on is the correct one.

\section{Future work}
\label{sec:future}

The per-sample reformulation turns the implicit penalty coefficient
$\beta_t^{(i)}$ of PPO-Clip into an explicit object. Its functional
form, a step function on the trust-region boundary in the present
case, can be modified by design without leaving the surrogate-loss
family of Section~\ref{sec:background}. We outline five extensions
that follow from making different choices for $\beta_t^{(i)}$. Each
defines a new policy-optimisation algorithm and is left as the
subject of a separate study.

\paragraph{Soft relaxations of the boundary.}
The coefficient defined in \eqref{eq:beta-def} is discontinuous at
$w_t^{(i)} = 1 \pm \epsilon$. Several proposals in the literature
soften this discontinuity from different angles. Trust Region-Guided
PPO \citep{wang2019trgppo} keeps the hard clip but lets its width
depend on the local KL of the policy. Truly PPO
\citep{wang2020trulyppo} keeps the clip and adds a rollback term
that drags the policy back when the ratio exits the trust region.
ESPO \citep{sun2022espo} removes ratio clipping altogether and
controls the inner loop by early stopping. Simple Policy
Optimization \citep{xie2025spo} substitutes clipping with a
regulariser on the ratio that admits a tighter trust region.
Probability Smoothing Policy Optimisation \citep{dwyer2025pspo}, in
the language-model setting, replaces the hard ratio with a soft
mixture of old and new policies so that the gradient is non-zero
everywhere.

The per-sample form unifies these proposals under one quantity: each
is a choice of $\beta_t^{(i)}$ in the template
$\mathbb{E}_t[\,w_t \hat{A}_t + \beta_t \log\pi_{\theta'}(a_t \mid
s_t)\,]$, with the soft variants replacing the step of
\eqref{eq:beta-def} by a continuous shape that agrees with it far
inside and far outside the trust region. The linear ramp of width
$\delta \ge 0$, which interpolates between $0$ and
$-w_t^{(i)}\hat{A}_t^{(i)}$ across the boundary, is one explicit
member: it recovers the PPO-Clip step as $\delta \to 0$ and the
unconstrained surrogate as $\delta \to \infty$. The right shape, as a
function of the task and of the inner-loop epoch count $K$, is left as
an empirical question.

\paragraph{Position-aware coefficient for sequence models.}
On token-level applications such as language-model fine-tuning, each
sample corresponds to one token of a generated completion. The
per-sample coefficient $\beta_t^{(i)}$ can then be allowed to depend
on the token's position within the sequence. A coefficient that is
sharper near the answer span and softer in the reasoning prefix, for
instance, is a specific position-conditioned $\beta_t^{(i)}$ and is
not expressible in standard PPO-Clip, which uses the same
trust-region radius for every token. The construction and empirical
study of position-aware variants is a direct use of the framework
and is left to follow-up work.

\paragraph{Age-conditioned coefficient for off-policy learning.}
The trust-region argument behind PPO-Clip assumes that the rollout
was sampled under a behaviour policy close to the current one. With
a replay buffer this assumption fails. The importance ratio
$w_t^{(i)}$ on a sample drawn from an older policy can be
arbitrarily large, and PPO-Clip discards all such samples by the
step coefficient. Letting $\beta_t^{(i)}$ depend on the age of the
sample at the time of the update, so that older samples are weighted
differently from fresh ones, defines an off-policy variant of the
algorithm whose properties can be analysed within the same
formulation. The appropriate functional form of the age dependence,
and its interaction with the bias-variance tradeoff of off-policy
estimation, is an open question.

\paragraph{Asymmetric trust regions.}
The coefficient \eqref{eq:beta-def} is symmetric under the swap
$(w_t^{(i)} - 1) \mapsto -(w_t^{(i)} - 1)$, in the sense that the
same step shape applies on both sides of the trust region. An
asymmetric variant softens the coefficient on the side that pulls
the policy back toward the rollout distribution and keeps it sharp
on the side that pushes it away. This modification is a single
change to the per-sample form and is not expressible inside the
original $\min$ formulation of \citet{schulman2017ppo}. Its analysis
is a natural follow-up.

\paragraph{A unified per-sample template.}
The per-sample form is not specific to PPO-Clip. The template
\begin{equation}
\mathcal{L}_{\mathrm{tmpl}}(\theta')
\;=\;
\mathbb{E}_t\!\bigl[\,w_t \hat{A}_t + \beta_t \log\pi_{\theta'}(a_t \mid s_t)\,\bigr]
\label{eq:template}
\end{equation}
recovers several existing on-policy algorithms once the dependence of
$\beta_t^{(i)}$ on the sample is specified, and Table~\ref{tab:template-instances}
summarises this view. The unconstrained importance-sampled surrogate
is the case $\beta_t \equiv 0$. PPO-KL with a fixed scalar
\citep{schulman2017ppo} corresponds to $\beta_t \equiv \beta \in
\mathbb{R}$. Adaptive PPO-KL \citep{schulman2017ppo} replaces the
constant by a quantity $\beta(t)$ that is updated once per rollout
according to the measured KL. PPO-Clip is the per-sample step
identified in Theorem~\ref{thm:identity}. Token-level PPO-Clip and
GRPO \citep{shao2024deepseekmath, deepseekai2025r1} apply the same
step independently to each token of a generated sequence. The
directions in this section correspond to making different choices of
$\beta_t^{(i)}$ within the same template.

\begin{table}[h]
\centering
\footnotesize
\renewcommand{\arraystretch}{1.35}
\resizebox{\textwidth}{!}{%
\begin{tabular}{@{} l l l l @{}}
\toprule
\textbf{Algorithm} & $\boldsymbol{\beta_t^{(i)}}$ \textbf{form} & \textbf{Depends on} & \textbf{Reference} \\
\midrule
Unconstrained surrogate & $0$ & --- & --- \\
PPO-KL (fixed) & $\beta \in \mathbb{R}$ & none & \citet{schulman2017ppo} \\
PPO-KL (adaptive) & $\beta(t) \in \mathbb{R}$ & training time & \citet{schulman2017ppo} \\
PPO-Clip & $-w_t \hat{A}_t \cdot \indicator_{\mathcal{I}_{\mathrm{kill}}}$ & $(w, \hat{A})$ & \citet{schulman2017ppo} \\
Soft-clip (linear ramp) & $-w_t \hat{A}_t \cdot g_\delta(w, \hat{A})$ & $(w, \hat{A})$, $\delta$ & this paper, Sec.~\ref{sec:future} \\
Token-level / GRPO & step shape per token & token $(w, \hat{A})$ & \citet{shao2024deepseekmath, deepseekai2025r1} \\
Position-aware & per-token shape, conditioned on position & position $+ (w, \hat{A})$ & future \\
Off-policy & per-sample shape, conditioned on age & age $+ (w, \hat{A})$ & future \\
Asymmetric & non-symmetric in $\mathrm{sign}(w-1)$ & $(w, \hat{A})$ & future \\
\bottomrule
\end{tabular}%
}
\caption{Instances of the per-sample template
\eqref{eq:template}. Each row is a particular choice of the
dependence of $\beta_t^{(i)}$ on the sample; the rest of the loss
is shared.}
\label{tab:template-instances}
\end{table}

Several research questions follow from the per-sample template. The
published algorithms in the upper part of the table can be trained
under identical pipelines and compared on the same metrics, which
isolates the effect of the shape of $\beta_t^{(i)}$ from the
implementation choices that the literature has shown matter strongly
\citep{engstrom2020implementation, andrychowicz2021matters}. The
soft, position-aware, age-conditioned and asymmetric directions
described in the previous paragraphs enrich the arguments of
$\beta_t^{(i)}$, and the corresponding algorithms can be implemented
and evaluated inside the same pipeline. On the theory side, the
boundedness and monotone-improvement guarantees available for fixed
and scheduled scalars in the original PPO analysis
\citep{kakade2002approximately, schulman2015trpo} have direct
analogues at the per-sample level, and characterising which shapes
of $\beta_t^{(i)}$ preserve them is an open problem. The template
also crosses domains: the same form covers MuJoCo locomotion,
language-model fine-tuning
\citep{ouyang2022instructgpt, shao2024deepseekmath} and off-policy
regimes, and $\beta_t^{(i)}$ is the shared design variable across
the three.

\small
\bibliographystyle{plainnat}
\bibliography{references}
\normalsize

\clearpage


\appendix
\raggedbottom
\makeatletter
\setlength{\@fptop}{0pt}
\makeatother

\section{Weight-space dual: PPO-Clip as a $\Phi$ penalty}
\label{app:phi}

The per-sample $\beta_t$ identity of the main text places PPO-Clip in
distribution space: the penalty is $\beta_t \log \pi_{\theta'}(a_t \mid
s_t)$, a per-sample KL contribution. PPO-Clip also admits a dual
formulation in weight space, where the penalty acts directly on the
deviation of $w_t$ from the trust region $[1-\epsilon,\, 1+\epsilon]$.

\begin{theorem}[Weight-space form of PPO-Clip]
\label{thm:phi}
The PPO-Clip objective can be written as
\[
L_{\mathrm{CLIP}}(\theta')
\;=\;
\frac{1}{N}\sum_{i=1}^{N}\sum_{t=1}^{H}
  \Bigl[\,w_t^{(i)}\hat{A}_t^{(i)} \;-\; \Phi\!\left(w_t^{(i)}, \hat{A}_t^{(i)}\right)\Bigr]
\]
with the per-sample weight-space penalty
\[
\Phi(w, \hat{A})
\;=\;
\begin{cases}
\bigl(w - (1+\epsilon)\bigr)\,\hat{A} & \text{if } w > 1+\epsilon \text{ and } \hat{A} > 0, \\[3pt]
\bigl(w - (1-\epsilon)\bigr)\,\hat{A} & \text{if } w < 1-\epsilon \text{ and } \hat{A} < 0, \\[3pt]
0 & \text{otherwise.}
\end{cases}
\]
\end{theorem}

\begin{proof}
Fix a sample $(i, t)$ and write $\ell_t^{(i)} =
\min\!\bigl(w_t^{(i)}\hat{A}_t^{(i)},\, \mathrm{clip}(w_t^{(i)})
\hat{A}_t^{(i)}\bigr)$ for its PPO-Clip per-sample contribution.

\emph{Case $w_t^{(i)} > 1+\epsilon$ and $\hat{A}_t^{(i)} > 0$.} The
clipped weight is $1+\epsilon$ and the clipped product is smaller than
the unclipped one, so
\[
\ell_t^{(i)} \;=\; (1+\epsilon)\hat{A}_t^{(i)}
\;=\; w_t^{(i)}\hat{A}_t^{(i)} - \bigl(w_t^{(i)} - (1+\epsilon)\bigr)\hat{A}_t^{(i)}
\;=\; w_t^{(i)}\hat{A}_t^{(i)} - \Phi(w_t^{(i)}, \hat{A}_t^{(i)}).
\]

\emph{Case $w_t^{(i)} < 1-\epsilon$ and $\hat{A}_t^{(i)} < 0$.} The
clipped weight is $1-\epsilon$ and again the clipped product is the
smaller, so
\[
\ell_t^{(i)} \;=\; (1-\epsilon)\hat{A}_t^{(i)}
\;=\; w_t^{(i)}\hat{A}_t^{(i)} - \bigl(w_t^{(i)} - (1-\epsilon)\bigr)\hat{A}_t^{(i)}
\;=\; w_t^{(i)}\hat{A}_t^{(i)} - \Phi(w_t^{(i)}, \hat{A}_t^{(i)}).
\]

\emph{Otherwise.} The unclipped term is the minimum, so $\ell_t^{(i)}
= w_t^{(i)}\hat{A}_t^{(i)} = w_t^{(i)}\hat{A}_t^{(i)} - 0$, and the
definition of $\Phi$ gives $\Phi(w_t^{(i)}, \hat{A}_t^{(i)}) = 0$.

Summing over $(i, t)$ yields the result.
\end{proof}

The penalty $\Phi$ is non-negative on $\mathcal{I}_{\mathrm{kill}}$:
when $w > 1+\epsilon$ and $\hat{A} > 0$ the factor $w - (1+\epsilon)$
is positive and so is $\hat{A}$; when $w < 1-\epsilon$ and $\hat{A} <
0$ both factors are negative and their product is positive again.
PPO-Clip therefore subtracts a non-negative penalty from the
unconstrained surrogate, proportional to how far $w_t^{(i)}$ exceeds
the trust-region boundary, on exactly the samples in
$\mathcal{I}_{\mathrm{kill}}$.

Combining the weight-space identity above with the per-sample
$\beta_t$ gradient identity of the main text, PPO-Clip admits three
forms on every minibatch, the first two equal in value and the third
equal in gradient:
\begin{enumerate}
\item the $\min$ formulation of \citet{schulman2017ppo}:
\[
L_{\mathrm{CLIP}} \;=\; \mathbb{E}_t\!\bigl[\min\!\bigl(w_t \hat{A}_t,\, \mathrm{clip}(w_t)\hat{A}_t\bigr)\bigr];
\]
\item the weight-space form of Theorem~\ref{thm:phi}:
\[
L_{\mathrm{CLIP}} \;=\; \mathbb{E}_t\!\bigl[\,w_t \hat{A}_t \;-\; \Phi(w_t, \hat{A}_t)\bigr],
\]
with penalty acting on $|w_t - 1|$ in importance-weight space;
\item the per-sample KL form of the main theorem, which matches
PPO-Clip in gradient:
\[
\nabla_{\theta'} L_{\mathrm{CLIP}} \;=\; \nabla_{\theta'}\,\mathbb{E}_t\!\bigl[\,w_t \hat{A}_t \;+\; \beta_t \log \pi_{\theta'}(a_t \mid s_t)\bigr],
\qquad \beta_t = -w_t \hat{A}_t \cdot \indicator\!\left[(i,t) \in \mathcal{I}_{\mathrm{kill}}\right],
\]
with penalty acting in distribution space.
\end{enumerate}
The first two forms are equal as functions and differ only in surface
notation; the third matches them in per-sample gradient and places
PPO-Clip inside the PPO-KL family. The space in which the
trust region is expressed (importance-weight space for $\Phi$,
distribution space for $\beta_t$) changes the notation but not the
per-sample gradient, which is the same in all three forms.

\begin{table}[h]
\centering
\small
\renewcommand{\arraystretch}{1.5}
\begin{tabular}{@{} l l l l @{}}
\toprule
\textbf{Form} & \textbf{Per-sample loss term} & \textbf{Penalty acts in} & \textbf{Reference} \\
\midrule
$\min$ & $\min\!\bigl(w_t\hat{A}_t,\, \mathrm{clip}(w_t)\hat{A}_t\bigr)$ & --- & \cite{schulman2017ppo} \\
$\Phi$ & $w_t\hat{A}_t - \Phi(w_t, \hat{A}_t)$ & importance-weight space & Theorem~\ref{thm:phi} \\
$\beta_t$ & $w_t\hat{A}_t + \beta_t \log\pi_{\theta'}(a_t \mid s_t)$ & distribution space & main theorem \\
\bottomrule
\end{tabular}
\caption{Three forms of the PPO-Clip per-sample loss. All three produce the same per-sample gradient on every $(i, t)$.}
\label{tab:three-forms}
\end{table}

\clearpage
\section{Supplementary Figures}
\label{app:figures}

This appendix collects supplementary figures that illustrate the
geometry of PPO-Clip and PPO-KL and the per-sample equivalence between
them. The notation follows the main text, with $w_t$ the importance
ratio, $\hat{A}_t$ the GAE advantage estimate, and
$\mathcal{I}_{\text{in}}, \mathcal{I}_{\text{kill}},
\mathcal{I}_{\text{pass}}$ the partition of the minibatch.

\subsection{The clipping function}

\begin{figure}[H]
\centering
\begin{tikzpicture}[>=Stealth, scale=0.85, transform shape]

\fill[green!8] (2.0, -0.3) rectangle (5.0, 4.5);

\draw[->, line width=0.6pt] (-0.5, 0) -- (7.5, 0) node[right, font=\scriptsize] {$w = \dfrac{\pi_{\theta'}}{{\color{blue}\pi_\theta}}$};
\draw[->, line width=0.6pt] (0, -0.5) -- (0, 4.8) node[above, font=\scriptsize] {$w_c$};

\draw[gray!30, line width=0.5pt, densely dotted] (0, 0) -- (6.3, 4.4);

\draw[berkeleyblue, line width=2.0pt]
    (0.2, 1.4) -- (2.0, 1.4)
    -- (5.0, 3.5)
    -- (7.0, 3.5);

\draw (2.0, -0.1) -- (2.0, 0.1);
\node[font=\tiny, below] at (2.0, -0.2) {$1\!-\!\epsilon$};
\draw (3.5, -0.1) -- (3.5, 0.1);
\node[font=\tiny, below] at (3.5, -0.2) {$1$};
\draw (5.0, -0.1) -- (5.0, 0.1);
\node[font=\tiny, below] at (5.0, -0.2) {$1\!+\!\epsilon$};

\draw (-0.1, 1.4) -- (0.1, 1.4);
\node[font=\tiny, left] at (-0.15, 1.4) {$1\!-\!\epsilon$};
\draw (-0.1, 2.45) -- (0.1, 2.45);
\node[font=\tiny, left] at (-0.15, 2.45) {$1$};
\draw (-0.1, 3.5) -- (0.1, 3.5);
\node[font=\tiny, left] at (-0.15, 3.5) {$1\!+\!\epsilon$};

\draw[gray, line width=0.3pt, densely dotted] (2.0, 0) -- (2.0, 1.4);
\draw[gray, line width=0.3pt, densely dotted] (5.0, 0) -- (5.0, 3.5);
\draw[gray, line width=0.3pt, densely dotted] (0, 1.4) -- (2.0, 1.4);
\draw[gray, line width=0.3pt, densely dotted] (0, 3.5) -- (5.0, 3.5);

\node[font=\tiny, text=green!50!black, font=\tiny\bfseries] at (3.5, 4.2) {gradient active};
\node[font=\tiny, text=red!70!black] at (1.0, 0.7) {$\nabla_{\theta'} = 0$};
\node[font=\tiny, text=red!70!black] at (6.2, 4.0) {$\nabla_{\theta'} = 0$};

\draw[decorate, decoration={brace, amplitude=4pt, mirror}, line width=0.5pt] (2.0, -0.6) -- (5.0, -0.6);
\node[font=\tiny, below] at (3.5, -1.0) {weights free to vary};

\end{tikzpicture}
\caption{The clipping function $w_c = \mathrm{clip}(w, 1-\epsilon, 1+\epsilon)$. Inside the band $[1-\epsilon, 1+\epsilon]$ the clipped weight equals the true importance ratio and the gradient flows normally; outside the band the clipped weight is constant and the gradient with respect to $\theta'$ vanishes.}
\label{fig:app-clipping-function}
\end{figure}
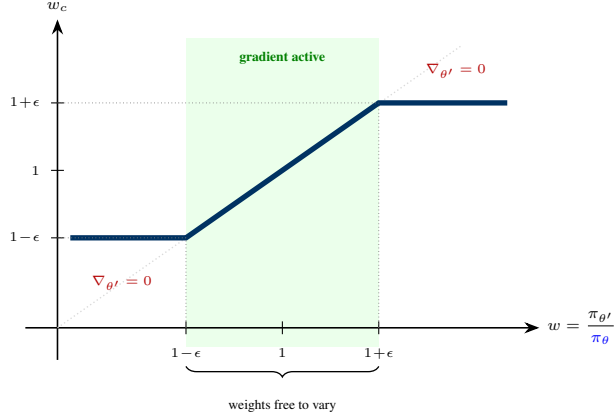

\subsection{The PPO-Clip surrogate}

\begin{figure}[H]
\centering
\begin{minipage}[t]{0.48\textwidth}
\centering
\begin{tikzpicture}[>=Stealth, scale=0.72, transform shape]

\node[font=\small\bfseries, text=green!50!black] at (4, 5.8) {Positive advantage ($\hat{A} > 0$)};

\draw[->, line width=0.6pt] (-0.3, 0) -- (8.5, 0) node[right, font=\scriptsize] {$w$};
\draw[->, line width=0.6pt] (0, -0.5) -- (0, 5.5) node[above, font=\scriptsize] {objective};

\draw (2.5, -0.1) -- (2.5, 0.1);
\node[font=\tiny, below] at (2.5, -0.15) {$1\!-\!\epsilon$};
\draw (4, -0.1) -- (4, 0.1);
\node[font=\tiny, below] at (4, -0.15) {$1$};
\draw (5.5, -0.1) -- (5.5, 0.1);
\node[font=\tiny, below] at (5.5, -0.15) {$1\!+\!\epsilon$};

\draw[gray, line width=0.3pt, densely dotted] (2.5, 0) -- (2.5, 5.2);
\draw[gray, line width=0.3pt, densely dotted] (5.5, 0) -- (5.5, 5.2);

\draw[red!70!black, line width=0.9pt, dashed] (0.5, 0.3) -- (8, 5.0);
\node[font=\tiny, text=red!70!black] at (7.0, 5.3) {$w\hat{A}$};

\draw[berkeleyblue, line width=0.9pt]
    (0.5, 1.55) -- (2.5, 1.55) -- (5.5, 3.43) -- (8, 3.43);
\node[font=\tiny, text=berkeleyblue] at (1.0, 1.9) {$w_c\hat{A}$};

\draw[green!50!black, line width=2.2pt, opacity=0.45]
    (0.5, 0.3) -- (5.5, 3.43) -- (8, 3.43);
\node[font=\tiny, text=green!50!black, fill=white, inner sep=1pt] at (7.0, 2.7) {$\min$ (PPO)};

\draw[<->, berkeleyblue, line width=0.5pt] (7.5, 3.53) -- (7.5, 4.7);
\node[font=\tiny, text=berkeleyblue, right] at (7.55, 4.15) {bounded};

\end{tikzpicture}
\end{minipage}
\hfill
\begin{minipage}[t]{0.48\textwidth}
\centering
\begin{tikzpicture}[>=Stealth, scale=0.72, transform shape]

\node[font=\small\bfseries, text=red!70!black] at (4, 5.8) {Negative advantage ($\hat{A} < 0$)};

\draw[->, line width=0.6pt] (-0.3, 0) -- (8.5, 0) node[right, font=\scriptsize] {$w$};
\draw[->, line width=0.6pt] (0, -0.5) -- (0, 5.5) node[above, font=\scriptsize] {objective};

\draw (2.5, -0.1) -- (2.5, 0.1);
\node[font=\tiny, below] at (2.5, -0.15) {$1\!-\!\epsilon$};
\draw (4, -0.1) -- (4, 0.1);
\node[font=\tiny, below] at (4, -0.15) {$1$};
\draw (5.5, -0.1) -- (5.5, 0.1);
\node[font=\tiny, below] at (5.5, -0.15) {$1\!+\!\epsilon$};

\draw[gray, line width=0.3pt, densely dotted] (2.5, 0) -- (2.5, 5.2);
\draw[gray, line width=0.3pt, densely dotted] (5.5, 0) -- (5.5, 5.2);

\draw[red!70!black, line width=0.9pt, dashed] (0.5, 4.8) -- (8, 0.1);
\node[font=\tiny, text=red!70!black] at (7.7, 0.6) {$w\hat{A}$};

\draw[berkeleyblue, line width=0.9pt]
    (0.5, 3.55) -- (2.5, 3.55) -- (5.5, 1.67) -- (8, 1.67);
\node[font=\tiny, text=berkeleyblue] at (7.3, 2.1) {$w_c\hat{A}$};

\draw[green!50!black, line width=2.2pt, opacity=0.45]
    (0.5, 3.55) -- (2.5, 3.55) -- (5.5, 1.67) -- (8, 0.1);
\node[font=\tiny, text=green!50!black, fill=white, inner sep=1pt] at (3.0, 0.9) {$\min$ (PPO)};

\draw[->, red!70!black, line width=0.5pt] (7.2, 1.1) -- (7.6, 0.3);
\node[font=\tiny, text=red!70!black, above] at (7.0, 1.15) {unbounded};

\end{tikzpicture}
\end{minipage}
\caption{The PPO-Clip surrogate, split by the sign of $\hat{A}$. Left panel ($\hat{A} > 0$): for $w > 1+\epsilon$ the minimum follows the clipped term, bounding the upside. Right panel ($\hat{A} < 0$): for $w > 1+\epsilon$ the minimum follows the unclipped term, so the penalty grows without bound. The asymmetry between the two panels is what makes PPO-Clip a one-sided trust region.}
\label{fig:app-ppo-clip-surrogate}
\end{figure}
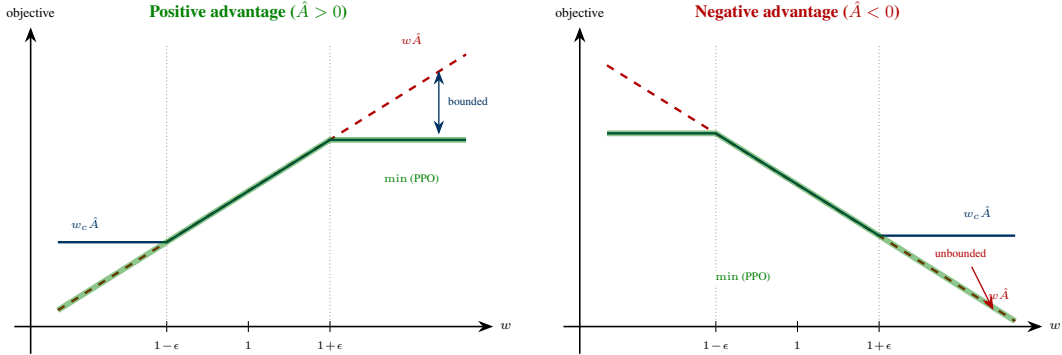

\subsection{PPO-Clip flowchart}

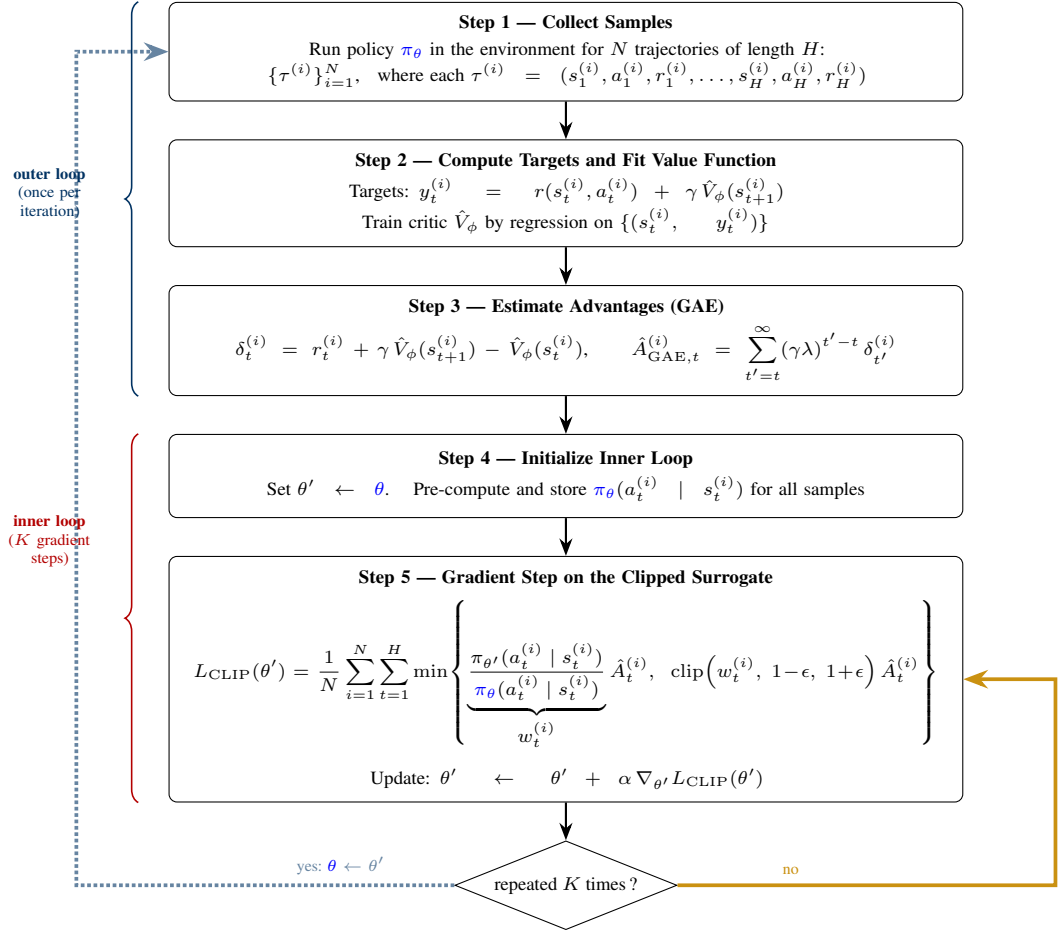
\begin{figure}[H]
\centering
\begin{tikzpicture}[
    >=Stealth,
    every node/.style={font=\scriptsize},
    box/.style={
        rectangle, rounded corners=3pt, minimum width=10.5cm, minimum height=1.1cm,
        text width=10.0cm, align=center, inner sep=5pt
    },
    decision/.style={
        diamond, draw, inner sep=2pt,
        font=\scriptsize, aspect=2.5, align=center
    }
]

\node[box, fill=white, draw=black] (P1) {
    \textbf{Step 1 --- Collect Samples}\\[2pt]
    Run policy ${\color{blue}\pi_\theta}$ in the environment for $N$ trajectories of length $H$:\\
    $\{\tau^{(i)}\}_{i=1}^N$, \; where each $\tau^{(i)} = (s_1^{(i)}, a_1^{(i)}, r_1^{(i)}, \ldots, s_H^{(i)}, a_H^{(i)}, r_H^{(i)})$
};

\node[box, fill=white, draw=black, below=0.5cm of P1] (P2) {
    \textbf{Step 2 --- Compute Targets and Fit Value Function}\\[2pt]
    Targets: $y_t^{(i)} = r(s_t^{(i)}, a_t^{(i)}) + \gamma\,\hat{V}_\phi(s_{t+1}^{(i)})$\\
    Train critic $\hat{V}_\phi$ by regression on $\{(s_t^{(i)},\; y_t^{(i)})\}$
};

\node[box, fill=white, draw=black, below=0.5cm of P2] (P3) {
    \textbf{Step 3 --- Estimate Advantages (GAE)}\\[2pt]
    $\delta_t^{(i)} = r_t^{(i)} + \gamma\,\hat{V}_\phi(s_{t+1}^{(i)}) - \hat{V}_\phi(s_t^{(i)})$, \qquad
    $\hat{A}_{\mathrm{GAE},t}^{(i)} = \displaystyle\sum_{t'=t}^{\infty}(\gamma\lambda)^{t'-t}\,\delta_{t'}^{(i)}$
};

\node[box, fill=white, draw=black, below=0.5cm of P3] (P4) {
    \textbf{Step 4 --- Initialize Inner Loop}\\[2pt]
    Set $\theta' \leftarrow {\color{blue}\theta}$. \; Pre-compute and store ${\color{blue}\pi_\theta}(a_t^{(i)} \mid s_t^{(i)})$ for all samples
};

\node[box, fill=white, draw=black, minimum height=2.8cm, text width=10.0cm, below=0.5cm of P4] (P5) {
    \textbf{Step 5 --- Gradient Step on the Clipped Surrogate}\\[4pt]
    $\displaystyle L_{\mathrm{CLIP}}(\theta') = \frac{1}{N}\sum_{i=1}^{N}\sum_{t=1}^{H} \min\!\left\{
        \underbrace{\frac{\pi_{\theta'}(a_t^{(i)} \mid s_t^{(i)})}{{\color{blue}\pi_\theta}(a_t^{(i)} \mid s_t^{(i)})}}_{\displaystyle w_t^{(i)}} \hat{A}_t^{(i)}, \;\;
        \mathrm{clip}\!\left(w_t^{(i)},\; 1\!-\!\epsilon,\; 1\!+\!\epsilon\right) \hat{A}_t^{(i)}
    \right\}$\\[6pt]
    Update: $\theta' \leftarrow \theta' + \alpha\,\nabla_{\theta'} L_{\mathrm{CLIP}}(\theta')$
};

\node[decision, draw=black, fill=white, below=0.5cm of P5] (DK) {repeated $K$ times\,?};

\draw[->, line width=0.8pt] (P1) -- (P2);
\draw[->, line width=0.8pt] (P2) -- (P3);
\draw[->, line width=0.8pt] (P3) -- (P4);
\draw[->, line width=0.8pt] (P4) -- (P5);
\draw[->, line width=0.8pt] (P5) -- (DK);

\draw[->, berkeleygold!80!black, line width=1.4pt]
    (DK.east) -- ++(5.0cm, 0) node[pos=0.3, above, font=\tiny, text=berkeleygold!80!black] {no} |- (P5.east);

\draw[->, berkeleyblue!60, line width=1.4pt, densely dotted]
    (DK.west) -- ++(-5.0cm, 0) node[pos=0.3, above, font=\tiny, text=berkeleyblue!60] {yes: ${\color{blue}\theta} \leftarrow \theta'$} |- (P1.west);

\draw[decorate, decoration={brace, amplitude=5pt, mirror}, berkeleyblue, line width=0.6pt]
    ([xshift=-0.4cm]P1.north west) -- ([xshift=-0.4cm]P3.south west);
\node[font=\tiny, text=berkeleyblue, anchor=east, text width=1.5cm, align=center] at ([xshift=-0.7cm]P2.west) {\textbf{outer loop}\\(once per\\iteration)};

\draw[decorate, decoration={brace, amplitude=5pt, mirror}, red!70!black, line width=0.6pt]
    ([xshift=-0.4cm]P4.north west) -- ([xshift=-0.4cm]P5.south west);
\node[font=\tiny, text=red!70!black, anchor=east, text width=1.5cm, align=center] at ([xshift=-0.7cm, yshift=-0.3cm]P4.south west) {\textbf{inner loop}\\($K$ gradient\\steps)};

\end{tikzpicture}
\caption{PPO-Clip. The outer loop collects rollouts and fits the critic; the inner loop takes $K$ gradient steps on the clipped surrogate $L_{\mathrm{CLIP}}$ before refreshing the behaviour policy.}
\label{fig:app-ppo-clip-flow}
\end{figure}

\subsection{PPO-KL flowchart}

\begin{figure}[H]
\centering
\begin{tikzpicture}[
    >=Stealth,
    every node/.style={font=\scriptsize},
    box/.style={
        rectangle, rounded corners=3pt, minimum width=10.5cm, minimum height=1.1cm,
        text width=10.0cm, align=center, inner sep=5pt
    },
    decision/.style={
        diamond, draw, inner sep=2pt,
        font=\scriptsize, aspect=2.5, align=center
    }
]

\node[box, fill=white, draw=black] (Q1) {
    \textbf{Step 1 --- Collect Samples}\\[2pt]
    Run policy ${\color{blue}\pi_\theta}$ in the environment for $N$ trajectories of length $H$:\\
    $\{\tau^{(i)}\}_{i=1}^N$, \; where each $\tau^{(i)} = (s_1^{(i)}, a_1^{(i)}, r_1^{(i)}, \ldots, s_H^{(i)}, a_H^{(i)}, r_H^{(i)})$
};

\node[box, fill=white, draw=black, below=0.4cm of Q1] (Q2) {
    \textbf{Step 2 --- Compute Targets}\\[2pt]
    $y_t^{(i)} = r(s_t^{(i)}, a_t^{(i)}) + \gamma\,\hat{V}_\phi(s_{t+1}^{(i)})$
};

\node[box, fill=white, draw=black, minimum height=1.4cm, below=0.4cm of Q2] (Q3) {
    \textbf{Step 3 --- Fit Value Function}\\[2pt]
    Train critic $\hat{V}_\phi$ by regression: $\;\phi \leftarrow \phi - \alpha_\phi \nabla_\phi \frac{1}{NH}\displaystyle\sum_{i=1}^{N}\sum_{t=1}^{H}\bigl(\hat{V}_\phi(s_t^{(i)}) - y_t^{(i)}\bigr)^2$
};

\node[box, fill=white, draw=black, below=0.4cm of Q3] (Q4) {
    \textbf{Step 4 --- Estimate Advantages (GAE)}\\[2pt]
    $\delta_t^{(i)} = r_t^{(i)} + \gamma\,\hat{V}_\phi(s_{t+1}^{(i)}) - \hat{V}_\phi(s_t^{(i)})$, \qquad
    $\hat{A}_{\mathrm{GAE},t}^{(i)} = \displaystyle\sum_{t'=t}^{\infty}(\gamma\lambda)^{t'-t}\,\delta_{t'}^{(i)}$
};

\node[box, fill=white, draw=black, below=0.4cm of Q4] (Q5) {
    \textbf{Step 5 --- Initialize Inner Loop}\\[2pt]
    Set $\theta' \leftarrow {\color{blue}\theta}$. \; Pre-compute and store ${\color{blue}\pi_\theta}(a_t^{(i)} \mid s_t^{(i)})$ for all samples
};

\node[box, fill=white, draw=black, minimum height=1.8cm, text width=10.0cm, below=0.4cm of Q5] (Q6) {
    \textbf{Step 6 --- Gradient Step on the KL-Penalised Objective}\\[4pt]
    $\displaystyle \mathcal{L}_{\mathrm{KL}}(\theta') = \frac{1}{N}\sum_{i=1}^{N}\sum_{t=1}^{H}\left[
        \frac{\pi_{\theta'}(a_t^{(i)} \mid s_t^{(i)})}{{\color{blue}\pi_\theta}(a_t^{(i)} \mid s_t^{(i)})}\,\hat{A}_t^{(i)}
        + \beta \log \pi_{\theta'}(a_t^{(i)} \mid s_t^{(i)})\right]$\\[4pt]
    Update: $\theta' \leftarrow \theta' + \alpha\,\nabla_{\theta'} \mathcal{L}_{\mathrm{KL}}(\theta')$
};

\node[decision, draw=black, fill=white, below=0.4cm of Q6] (QDK) {repeated $K$ times\,?};

\node[box, fill=white, draw=black, below=0.4cm of QDK] (Q7) {
    \textbf{Step 7 --- Dual Variable Update}\\[2pt]
    $\beta \leftarrow \beta + \alpha_\beta\bigl(D_{\mathrm{KL}}({\color{blue}\pi_\theta} \| \pi_{\theta'}) - \epsilon\bigr)$
    \qquad {\footnotesize ($D_{\mathrm{KL}} > \epsilon \Rightarrow$ increase $\beta$; \; $D_{\mathrm{KL}} < \epsilon \Rightarrow$ decrease $\beta$)}
};

\node[box, fill=white, draw=black, below=0.4cm of Q7] (Q8) {
    \textbf{Step 8 --- Adopt New Policy}\\[2pt]
    ${\color{blue}\theta} \leftarrow \theta'$
};

\draw[->, line width=0.8pt] (Q1) -- (Q2);
\draw[->, line width=0.8pt] (Q2) -- (Q3);
\draw[->, line width=0.8pt] (Q3) -- (Q4);
\draw[->, line width=0.8pt] (Q4) -- (Q5);
\draw[->, line width=0.8pt] (Q5) -- (Q6);
\draw[->, line width=0.8pt] (Q6) -- (QDK);
\draw[->, line width=0.8pt] (QDK) -- (Q7) node[midway, right, font=\tiny] {yes};
\draw[->, line width=0.8pt] (Q7) -- (Q8);

\draw[->, berkeleygold!80!black, line width=1.4pt]
    (QDK.east) -- ++(4.5cm, 0) node[pos=0.3, above, font=\tiny, text=berkeleygold!80!black] {no} |- (Q6.east);

\draw[->, berkeleyblue!60, line width=1.4pt, densely dotted]
    (Q8.west) -- ++(-4.5cm, 0) node[pos=0.3, above, font=\tiny, text=berkeleyblue!60] {next iteration} |- (Q1.west);

\draw[decorate, decoration={brace, amplitude=5pt, mirror}, berkeleyblue, line width=0.6pt]
    ([xshift=-0.4cm]Q1.north west) -- ([xshift=-0.4cm]Q4.south west);
\node[font=\tiny, text=berkeleyblue, anchor=east, text width=1.5cm, align=center] at ([xshift=-0.7cm, yshift=0.5cm]Q3.west) {\textbf{outer loop}\\(once per\\iteration)};

\draw[decorate, decoration={brace, amplitude=5pt, mirror}, red!70!black, line width=0.6pt]
    ([xshift=-0.4cm]Q5.north west) -- ([xshift=-0.4cm]Q6.south west);
\node[font=\tiny, text=red!70!black, anchor=east, text width=1.5cm, align=center] at ([xshift=-0.7cm]Q5.south west) {\textbf{inner loop}\\($K$ gradient\\steps)};

\end{tikzpicture}
\caption{PPO-KL. The objective is the unclipped importance-weighted advantage plus a KL penalty with a scalar coefficient $\beta$; after the inner loop $\beta$ is adjusted by a dual gradient step that targets a desired KL.}
\label{fig:app-ppo-kl-flow}
\end{figure}
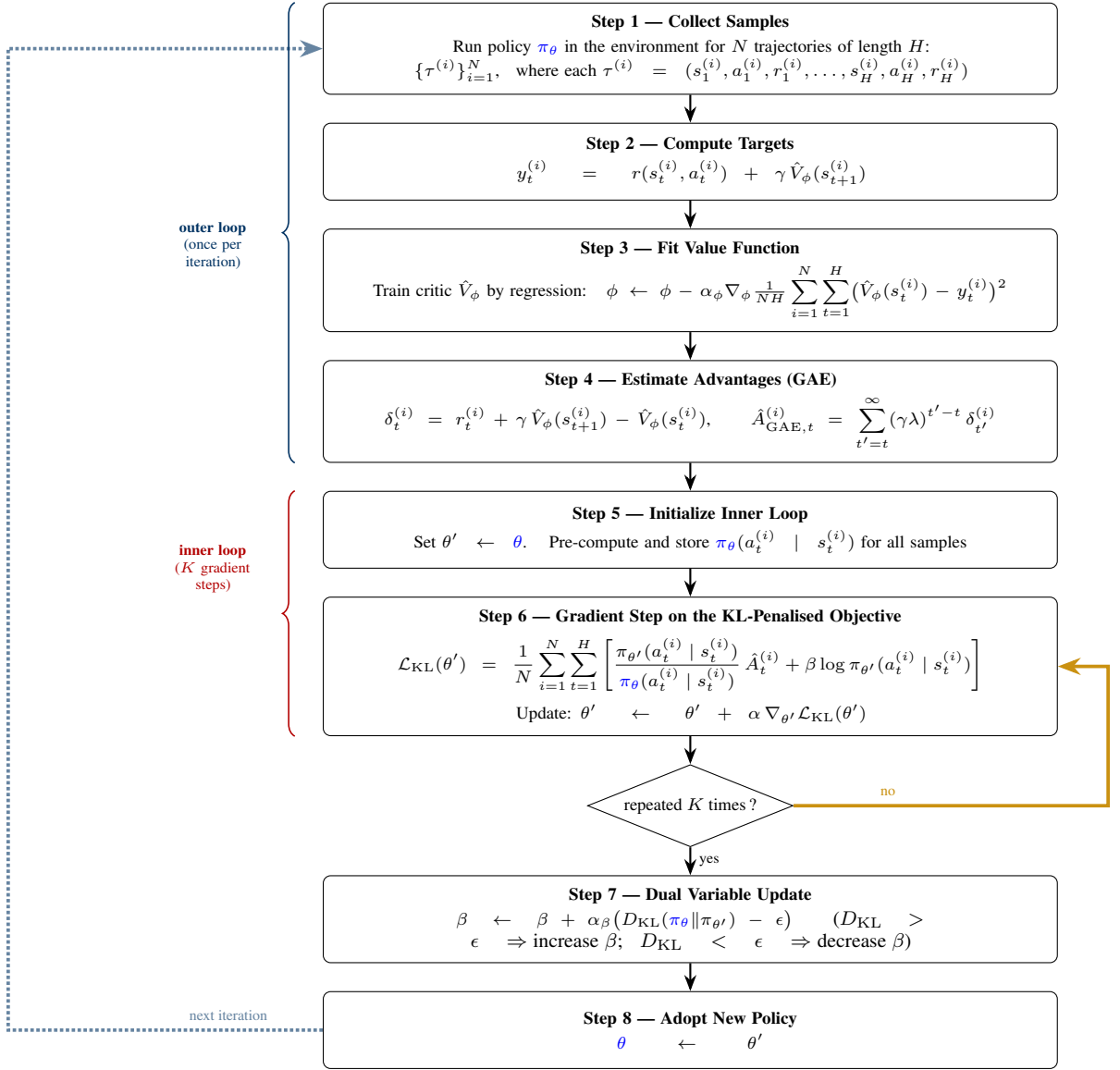

\subsection{Effective gradient multiplier as a function of $w$}

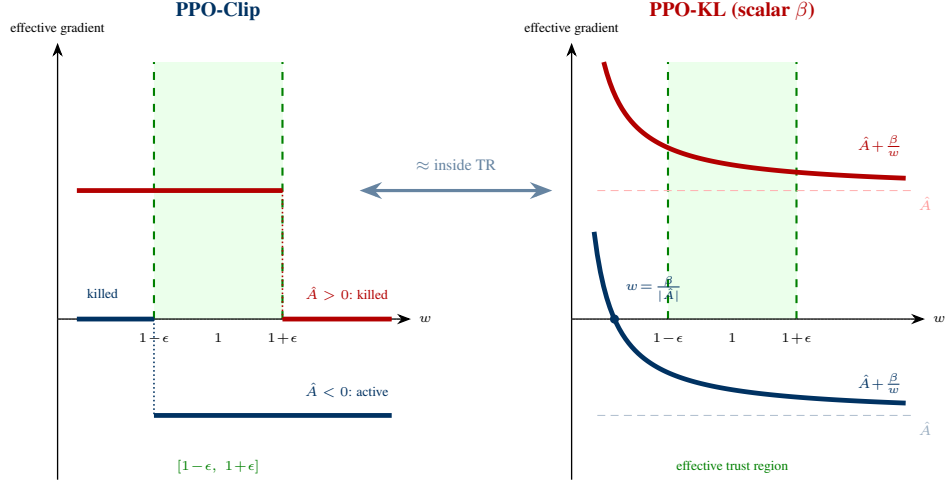
\begin{figure}[H]
\centering
\begin{tikzpicture}[>=Stealth, scale=0.85, transform shape]

\begin{scope}[shift={(-5,0)}]
\node[font=\small\bfseries, text=berkeleyblue] at (2.5, 4.8) {PPO-Clip};

\fill[green!8] (1.5, 0) rectangle (3.5, 4.0);
\draw[green!50!black, line width=0.8pt, dashed] (1.5, 0) -- (1.5, 4.0);
\draw[green!50!black, line width=0.8pt, dashed] (3.5, 0) -- (3.5, 4.0);

\draw[->, line width=0.5pt] (0, 0) -- (5.5, 0) node[right, font=\tiny] {$w$};
\draw[->, line width=0.5pt] (0, -2.5) -- (0, 4.3) node[above, font=\tiny] {effective gradient};

\draw[gray, line width=0.3pt, densely dotted] (0, 0) -- (5.3, 0);

\node[font=\tiny, below] at (1.5, -0.1) {$1\!-\!\epsilon$};
\node[font=\tiny, below] at (2.5, -0.1) {$1$};
\node[font=\tiny, below] at (3.5, -0.1) {$1\!+\!\epsilon$};

\draw[red!70!black, line width=1.8pt]
    (0.3, 2.0) -- (3.5, 2.0);
\draw[red!70!black, line width=1.8pt]
    (3.5, 0) -- (5.2, 0);
\draw[red!70!black, line width=0.6pt, densely dotted] (3.5, 2.0) -- (3.5, 0);
\node[font=\tiny, text=red!70!black] at (4.5, 0.4) {$\hat{A}>0$: killed};

\draw[berkeleyblue, line width=1.8pt]
    (0.3, 0) -- (1.5, 0);
\draw[berkeleyblue, line width=1.8pt]
    (1.5, -1.5) -- (5.2, -1.5);
\draw[berkeleyblue, line width=0.6pt, densely dotted] (1.5, 0) -- (1.5, -1.5);
\node[font=\tiny, text=berkeleyblue] at (4.5, -1.1) {$\hat{A}<0$: active};
\node[font=\tiny, text=berkeleyblue] at (0.7, 0.4) {killed};

\node[font=\tiny, text=green!50!black] at (2.5, -2.3) {$[1\!-\!\epsilon,\; 1\!+\!\epsilon]$};
\end{scope}

\begin{scope}[shift={(3,0)}]
\node[font=\small\bfseries, text=red!70!black] at (2.5, 4.8) {PPO-KL (scalar $\beta$)};

\fill[green!8] (1.5, 0) rectangle (3.5, 4.0);
\draw[green!50!black, line width=0.8pt, dashed] (1.5, 0) -- (1.5, 4.0);
\draw[green!50!black, line width=0.8pt, dashed] (3.5, 0) -- (3.5, 4.0);

\draw[->, line width=0.5pt] (0, 0) -- (5.5, 0) node[right, font=\tiny] {$w$};
\draw[->, line width=0.5pt] (0, -2.5) -- (0, 4.3) node[above, font=\tiny] {effective gradient};

\draw[gray, line width=0.3pt, densely dotted] (0, 0) -- (5.3, 0);

\node[font=\tiny, below] at (1.5, -0.1) {$1\!-\!\epsilon$};
\node[font=\tiny, below] at (2.5, -0.1) {$1$};
\node[font=\tiny, below] at (3.5, -0.1) {$1\!+\!\epsilon$};

\draw[red!30, line width=0.4pt, densely dashed] (0.4, 2.0) -- (5.3, 2.0);
\draw[berkeleyblue!30, line width=0.4pt, densely dashed] (0.4, -1.5) -- (5.3, -1.5);
\node[font=\tiny, text=red!40] at (5.5, 1.8) {$\hat{A}$};
\node[font=\tiny, text=berkeleyblue!40] at (5.5, -1.7) {$\hat{A}$};

\draw[red!70!black, line width=1.8pt]
    plot[variable=\x, domain=0.5:5.2, samples=100] (\x, {2.0 + 1/\x});
\node[font=\tiny, text=red!70!black] at (4.8, 2.7) {$\hat{A}\!+\!\tfrac{\beta}{w}$};

\draw[berkeleyblue, line width=1.8pt]
    plot[variable=\x, domain=0.35:5.2, samples=100] (\x, {-1.5 + 1/\x});
\node[font=\tiny, text=berkeleyblue] at (4.8, -1.0) {$\hat{A}\!+\!\tfrac{\beta}{w}$};

\fill[berkeleyblue] (0.667, 0) circle (2pt);
\node[font=\tiny, text=berkeleyblue, above right] at (0.72, 0.15) {$w\!=\!\tfrac{\beta}{|\hat{A}|}$};

\node[font=\tiny, text=green!50!black] at (2.5, -2.3) {effective trust region};
\end{scope}

\draw[<->, line width=1.2pt, berkeleyblue!60] (-0.3, 2.0) -- (2.7, 2.0);
\node[font=\scriptsize, text=berkeleyblue!60] at (1.2, 2.4) {$\approx$ inside TR};

\end{tikzpicture}
\caption{Per-sample effective gradient multiplier as a function of the importance ratio $w$. Left: PPO-Clip is a hard step. The multiplier equals $\hat{A}$ inside $[1-\epsilon, 1+\epsilon]$ and on $\mathcal{I}_{\text{pass}}$, and zero on $\mathcal{I}_{\text{kill}}$. Right: a scalar PPO-KL produces the smooth curve $\hat{A} + \beta/w$ (the asymptote is $\hat{A}$). The two coincide inside the trust region; outside, PPO-Clip is a step function while scalar PPO-KL is a smooth one. The per-sample $\beta_t$ construction of the main text is the choice that recovers the step function.}
\label{fig:app-gradient-comparison}
\end{figure}

\subsection{Morphing from scalar to per-sample $\beta$}

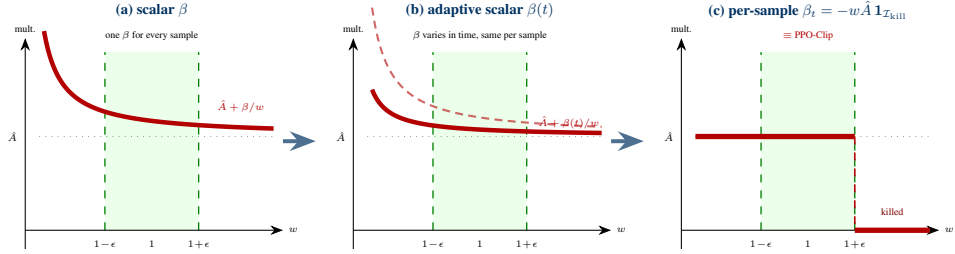
\begin{figure}[H]
\centering
\begin{tikzpicture}[scale=0.62, transform shape, >=Stealth]

\begin{scope}[shift={(-9,0)}]
\node[font=\small\bfseries, text=berkeleyblue] at (2.7, 4.7) {(a) scalar $\beta$};
\node[font=\tiny] at (2.7, 4.15) {one $\beta$ for every sample};

\fill[green!8] (1.7, 0) rectangle (3.7, 3.8);
\draw[green!50!black, dashed, line width=0.5pt] (1.7, 0) -- (1.7, 3.8);
\draw[green!50!black, dashed, line width=0.5pt] (3.7, 0) -- (3.7, 3.8);

\draw[->] (0, 0) -- (5.5, 0) node[right, font=\tiny] {$w$};
\draw[->] (0, -0.4) -- (0, 4.1) node[above, font=\tiny] {mult.};
\draw[gray, dotted, line width=0.3pt] (0.3, 2.0) -- (5.3, 2.0);
\node[font=\tiny, left] at (-0.05, 2.0) {$\hat{A}$};

\node[font=\tiny, below] at (1.7, -0.1) {$1\!-\!\epsilon$};
\node[font=\tiny, below] at (2.7, -0.1) {$1$};
\node[font=\tiny, below] at (3.7, -0.1) {$1\!+\!\epsilon$};

\draw[red!70!black, line width=1.7pt]
    plot[variable=\x, domain=0.4:5.3, samples=120] (\x, {2.0 + 0.9/\x});
\node[font=\tiny, text=red!70!black] at (4.6, 2.65) {$\hat{A}+\beta/w$};
\end{scope}

\begin{scope}[shift={(-2,0)}]
\node[font=\small\bfseries, text=berkeleyblue] at (2.7, 4.7) {(b) adaptive scalar $\beta(t)$};
\node[font=\tiny] at (2.7, 4.15) {$\beta$ varies in time, same per sample};

\fill[green!8] (1.7, 0) rectangle (3.7, 3.8);
\draw[green!50!black, dashed, line width=0.5pt] (1.7, 0) -- (1.7, 3.8);
\draw[green!50!black, dashed, line width=0.5pt] (3.7, 0) -- (3.7, 3.8);

\draw[->] (0, 0) -- (5.5, 0) node[right, font=\tiny] {$w$};
\draw[->] (0, -0.4) -- (0, 4.1) node[above, font=\tiny] {mult.};
\draw[gray, dotted, line width=0.3pt] (0.3, 2.0) -- (5.3, 2.0);
\node[font=\tiny, left] at (-0.05, 2.0) {$\hat{A}$};

\node[font=\tiny, below] at (1.7, -0.1) {$1\!-\!\epsilon$};
\node[font=\tiny, below] at (2.7, -0.1) {$1$};
\node[font=\tiny, below] at (3.7, -0.1) {$1\!+\!\epsilon$};

\draw[red!70!black, line width=1.7pt]
    plot[variable=\x, domain=0.4:5.3, samples=120] (\x, {2.0 + 0.4/\x});
\draw[red!70!black, line width=0.8pt, densely dashed, opacity=0.6]
    plot[variable=\x, domain=0.4:5.3, samples=120] (\x, {2.0 + 1.1/\x});
\node[font=\tiny, text=red!70!black] at (4.6, 2.3) {$\hat{A}+\beta(t)/w$};
\end{scope}

\begin{scope}[shift={(5,0)}]
\node[font=\small\bfseries, text=berkeleyblue] at (2.7, 4.7) {(c) per-sample $\beta_t = -w\hat{A}\,\indicator_{\mathcal{I}_{\mathrm{kill}}}$};
\node[font=\tiny, text=red!70!black] at (2.7, 4.15) {$\equiv$ PPO-Clip};

\fill[green!8] (1.7, 0) rectangle (3.7, 3.8);
\draw[green!50!black, dashed, line width=0.5pt] (1.7, 0) -- (1.7, 3.8);
\draw[green!50!black, dashed, line width=0.5pt] (3.7, 0) -- (3.7, 3.8);

\draw[->] (0, 0) -- (5.5, 0) node[right, font=\tiny] {$w$};
\draw[->] (0, -0.4) -- (0, 4.1) node[above, font=\tiny] {mult.};
\draw[gray, dotted, line width=0.3pt] (0.3, 2.0) -- (5.3, 2.0);
\node[font=\tiny, left] at (-0.05, 2.0) {$\hat{A}$};

\node[font=\tiny, below] at (1.7, -0.1) {$1\!-\!\epsilon$};
\node[font=\tiny, below] at (2.7, -0.1) {$1$};
\node[font=\tiny, below] at (3.7, -0.1) {$1\!+\!\epsilon$};

\draw[red!70!black, line width=2pt] (0.3, 2.0) -- (3.7, 2.0);
\draw[red!70!black, line width=2pt] (3.7, 0) -- (5.3, 0);
\draw[red!70!black, line width=0.6pt, dashed] (3.7, 2.0) -- (3.7, 0);
\node[font=\tiny, text=red!50!black] at (4.5, 0.4) {killed};
\end{scope}

\draw[->, line width=1.5pt, berkeleyblue!70] (-3.5, 1.9) -- (-2.8, 1.9);
\draw[->, line width=1.5pt, berkeleyblue!70] (3.5, 1.9) -- (4.2, 1.9);

\end{tikzpicture}
\caption{Morphing from scalar to per-sample~$\beta$, shown for the case $\hat{A}>0$. (a)~A constant scalar $\beta$ produces the smooth curve $\hat{A}+\beta/w$, which can never coincide with the PPO-Clip step. (b)~Letting $\beta$ adapt in time but stay scalar shifts the curve along its family but cannot reshape it. (c)~Letting $\beta$ become per-sample, with $\beta_t = -w_t\hat{A}_t$ exactly on $\mathcal{I}_{\mathrm{kill}}$ and zero elsewhere, snaps the curve onto the PPO-Clip step function. The clip is the per-sample limit of the KL surrogate; scalar penalties never reach it.}
\label{fig:app-morph}
\end{figure}

\subsection{The $\beta_t$ map in the $(w, \hat{A})$ plane}

\begin{figure}[H]
\centering
\begin{tikzpicture}[scale=1.55, >=Stealth]

\fill[red!15] (0.5, 0) rectangle (3, 2);
\fill[red!25] (1.0, 0.4) rectangle (3, 2);
\fill[red!40] (1.5, 0.8) rectangle (3, 2);
\fill[red!55] (2.0, 1.2) rectangle (3, 2);
\fill[red!70] (2.5, 1.6) rectangle (3, 2);

\fill[red!15] (-3, -2) rectangle (-0.5, 0);
\fill[red!25] (-3, -2) rectangle (-1.0, -0.4);
\fill[red!40] (-3, -2) rectangle (-1.5, -0.8);
\fill[red!55] (-3, -2) rectangle (-2.0, -1.2);
\fill[red!70] (-3, -2) rectangle (-2.5, -1.6);

\fill[blue!8] (-3, 0) rectangle (-0.5, 2);
\fill[blue!8] (0.5, -2) rectangle (3, 0);

\fill[green!5] (-0.5, -2) rectangle (0.5, 2);

\draw[green!60!black, dashed, line width=0.9pt] (-0.5, -2.1) -- (-0.5, 2.1);
\draw[green!60!black, dashed, line width=0.9pt] (0.5, -2.1) -- (0.5, 2.1);

\draw[->, line width=0.7pt] (-3.4, 0) -- (3.4, 0) node[right, font=\footnotesize] {$w$};
\draw[->, line width=0.7pt] (0, -2.4) -- (0, 2.4) node[above, font=\footnotesize] {$\hat{A}$};

\node[font=\scriptsize, below right] at (0.05, -0.05) {$1$};
\node[font=\scriptsize, below=2pt] at (-0.5, -0.05) {$1\!-\!\epsilon$};
\node[font=\scriptsize, below=2pt] at (0.5, -0.05) {$1\!+\!\epsilon$};

\node[font=\footnotesize\bfseries, text=red!70!black] at (1.8, 1.85) {$\mathcal{I}_{\mathrm{kill}}$};
\node[font=\scriptsize, text=red!70!black, fill=white, fill opacity=0.7, text opacity=1, inner sep=1pt] at (1.8, 1.55) {$\beta_t = -w\hat{A}<0$};

\node[font=\footnotesize\bfseries, text=red!70!black] at (-1.8, -1.85) {$\mathcal{I}_{\mathrm{kill}}$};
\node[font=\scriptsize, text=red!70!black, fill=white, fill opacity=0.7, text opacity=1, inner sep=1pt] at (-1.8, -1.55) {$\beta_t = -w\hat{A}>0$};

\node[font=\footnotesize\bfseries, text=blue!50!black] at (-1.8, 1.85) {$\mathcal{I}_{\mathrm{pass}}$};
\node[font=\scriptsize, text=blue!50!black] at (-1.8, 1.55) {$\beta_t = 0$};

\node[font=\footnotesize\bfseries, text=blue!50!black] at (1.8, -1.85) {$\mathcal{I}_{\mathrm{pass}}$};
\node[font=\scriptsize, text=blue!50!black] at (1.8, -1.55) {$\beta_t = 0$};

\node[font=\footnotesize\bfseries, text=green!40!black] at (0, 1.85) {$\mathcal{I}_{\mathrm{in}}$};
\node[font=\scriptsize, text=green!40!black] at (0, 1.55) {$\beta_t\!=\!0$};

\end{tikzpicture}
\caption{The $\beta_t$ landscape in $(w, \hat{A})$ space. Vertical dashed lines mark the trust region $[1-\epsilon, 1+\epsilon]$. Two corners (top-right $w>1+\epsilon,\hat{A}>0$ and bottom-left $w<1-\epsilon,\hat{A}<0$) carry the per-sample coefficient $\beta_t = -w\hat{A}$, with shading indicating $|\beta_t|$. Everywhere else $\beta_t = 0$. The support and value of $\beta_t$ reproduce the PPO-Clip gradient sample by sample.}
\label{fig:app-beta-map}
\end{figure}
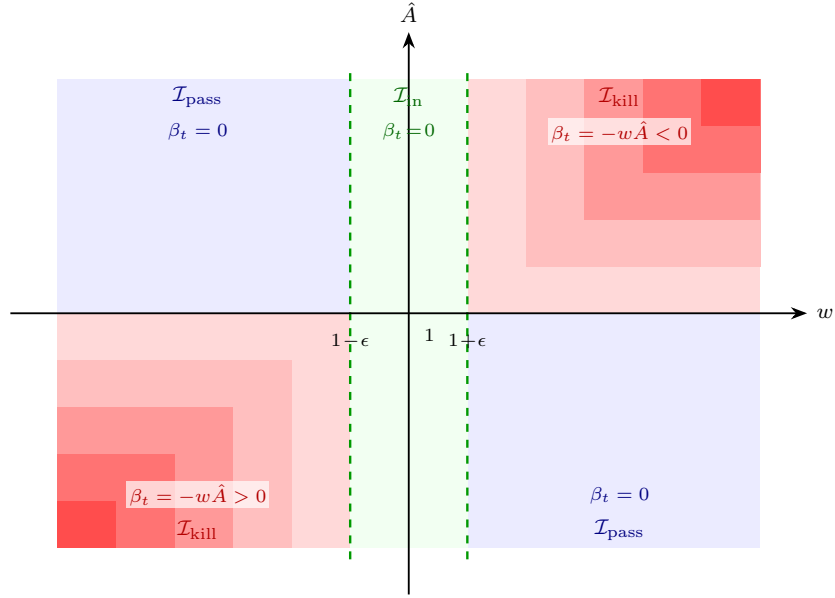

\subsection{The same per-sample gradient via two surrogates}

\begin{figure}[H]
\centering
\begin{tikzpicture}[
  >=Stealth,
  font=\small,
  loss/.style={
    rectangle, rounded corners=3pt, draw=black, line width=0.7pt,
    minimum width=6.4cm, minimum height=1.3cm, align=center, fill=white
  },
  sample/.style={
    rectangle, rounded corners=3pt, draw=black, line width=0.7pt,
    minimum width=5cm, minimum height=0.9cm, align=center, fill=white
  },
  result/.style={
    rectangle, rounded corners=3pt, draw=black, line width=0.7pt,
    minimum width=8.5cm, minimum height=1.4cm, align=center, fill=white
  }
]

\node[sample] (S) at (0, 5) {sample $(s_t, a_t, \hat{A}_t, w_t)$};

\node[loss] (LC) at (-4.2, 3) {
  \textbf{PPO-Clip per-sample term} \\[2pt]
  $\ell_t^{\mathrm{CLIP}} = \min\!\bigl(w_t \hat{A}_t,\, \mathrm{clip}(w_t)\hat{A}_t\bigr)$
};
\node[loss] (LK) at (4.2, 3) {
  \textbf{per-sample KL term} \\[2pt]
  $\ell_t^{\mathrm{KL}} = w_t \hat{A}_t + \beta_t\, \log \pi_{\theta'}(a_t \mid s_t)$
};

\draw[->, line width=0.7pt] (S.south) -- (LC.north);
\draw[->, line width=0.7pt] (S.south) -- (LK.north);

\node[result] (R) at (0, 0.4) {
  $g_t \;=\; \begin{cases}
    \hat{A}_t\, \nabla_{\theta'} w_t & \text{if } (i,t) \in \mathcal{I}_{\mathrm{in}} \cup \mathcal{I}_{\mathrm{pass}}, \\[2pt]
    0 & \text{if } (i,t) \in \mathcal{I}_{\mathrm{kill}}.
  \end{cases}$
};

\draw[->, line width=0.7pt] (LC.south) -- ++(0, -0.5) node[midway, left, font=\scriptsize] {$\nabla_{\theta'}$} -| ($(R.north) + (-2.5, 0)$);
\draw[->, line width=0.7pt] (LK.south) -- ++(0, -0.5) node[midway, right, font=\scriptsize] {$\nabla_{\theta'}$ with $\beta_t = -w_t\hat{A}_t\,\indicator_{\mathcal{I}_{\mathrm{kill}}}$} -| ($(R.north) + (2.5, 0)$);

\end{tikzpicture}
\caption{The per-sample gradient under the two surrogates. The same sample feeds both per-sample terms; differentiating each in $\theta'$ yields the same $g_t$. PPO-Clip selects the active branch of the $\min$ to kill the gradient on $\mathcal{I}_{\mathrm{kill}}$; the per-sample KL surrogate sets $\beta_t = -w_t \hat{A}_t$ on $\mathcal{I}_{\mathrm{kill}}$ and $0$ elsewhere, which makes the bracket $\hat{A}_t + \beta_t/w_t$ vanish on those samples.}
\label{fig:app-two-routes}
\end{figure}
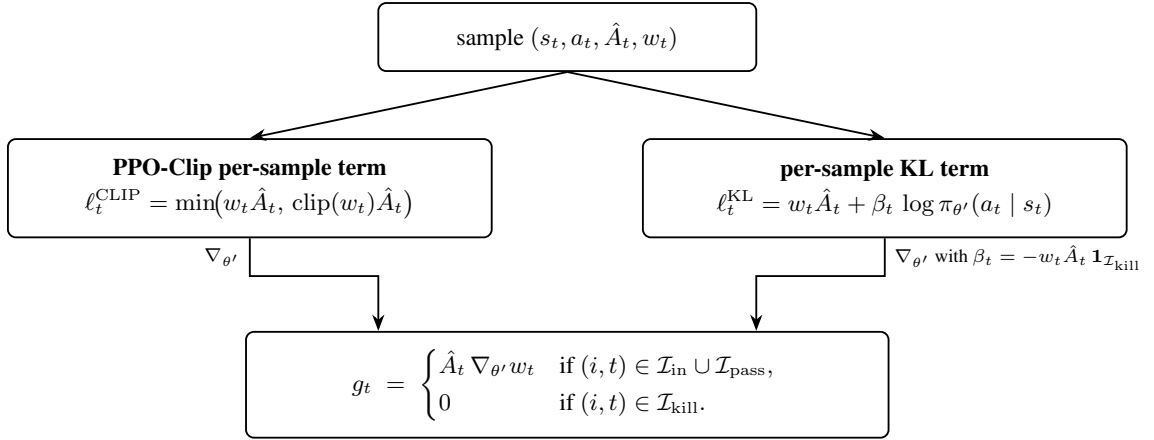

\clearpage
\section{Supplementary experimental results}
\label{app:exp}

\subsection{Per-task equivalence}

Figure~\ref{fig:identity} reports PPO-Clip and the per-sample PPO-KL
surrogate on each of the seven tasks separately, namely CartPole-v1,
LunarLander-v3, Hopper-v4, HalfCheetah-v4, Walker2d-v4, Ant-v4, and
Humanoid-v4. The two are indistinguishable on every environment. Each
curve is a mean over $5$ seeds with a $\pm$ std band.

\begin{figure}[H]
  \centering
  \begin{subfigure}[t]{0.49\textwidth}
    \centering
    \includegraphics[width=\textwidth]{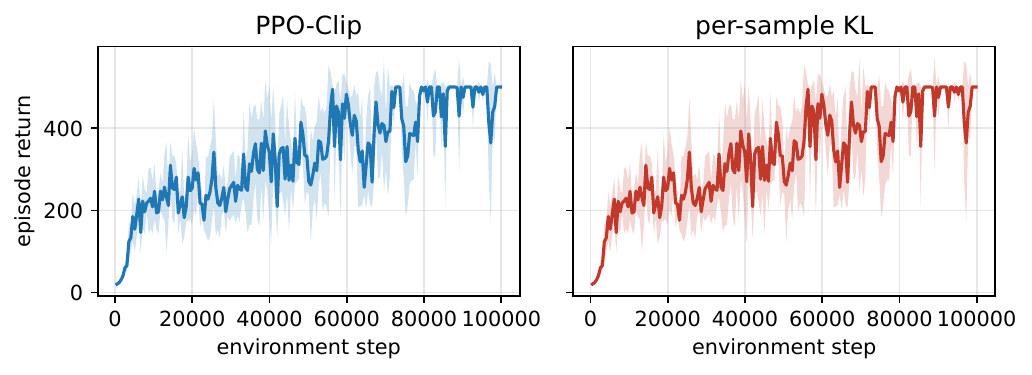}
    \caption{CartPole-v1}
    \label{fig:id-cartpole}
  \end{subfigure}
  \hfill
  \begin{subfigure}[t]{0.49\textwidth}
    \centering
    \includegraphics[width=\textwidth]{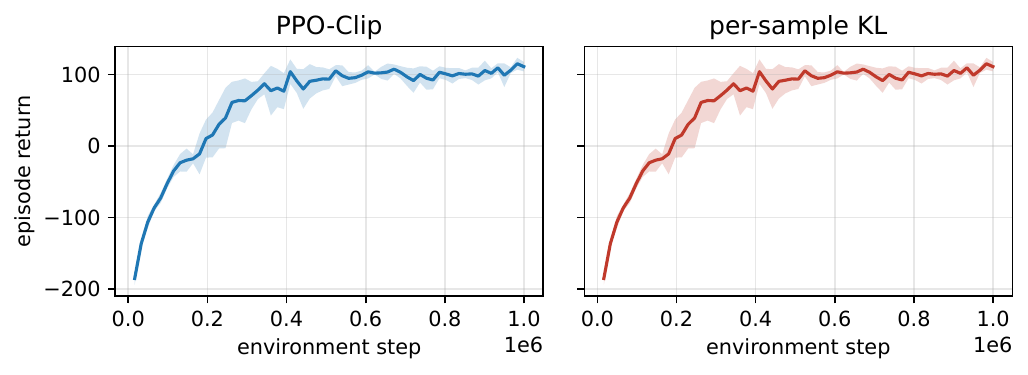}
    \caption{LunarLander-v3}
    \label{fig:id-lunarlander}
  \end{subfigure}

  \begin{subfigure}[t]{0.49\textwidth}
    \centering
    \includegraphics[width=\textwidth]{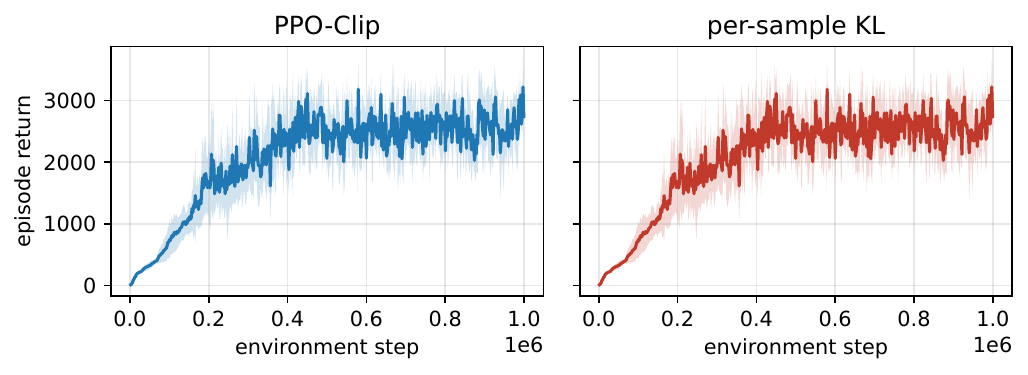}
    \caption{Hopper-v4}
    \label{fig:id-hopper}
  \end{subfigure}
  \hfill
  \begin{subfigure}[t]{0.49\textwidth}
    \centering
    \includegraphics[width=\textwidth]{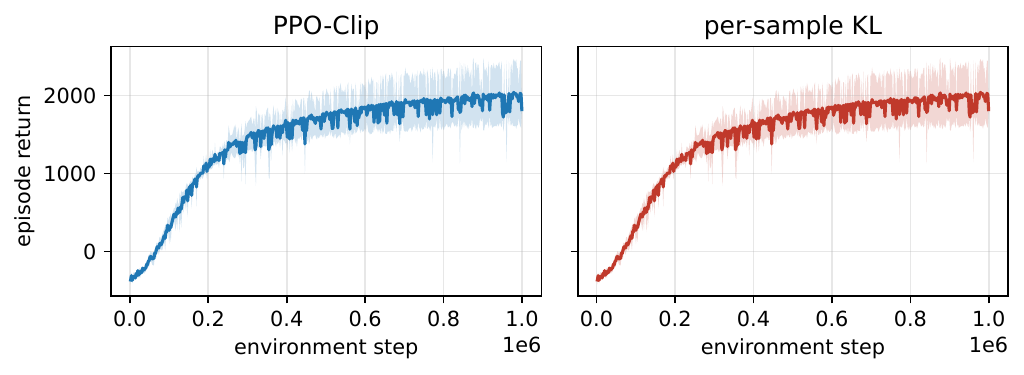}
    \caption{HalfCheetah-v4}
    \label{fig:id-halfcheetah}
  \end{subfigure}

  \begin{subfigure}[t]{0.49\textwidth}
    \centering
    \includegraphics[width=\textwidth]{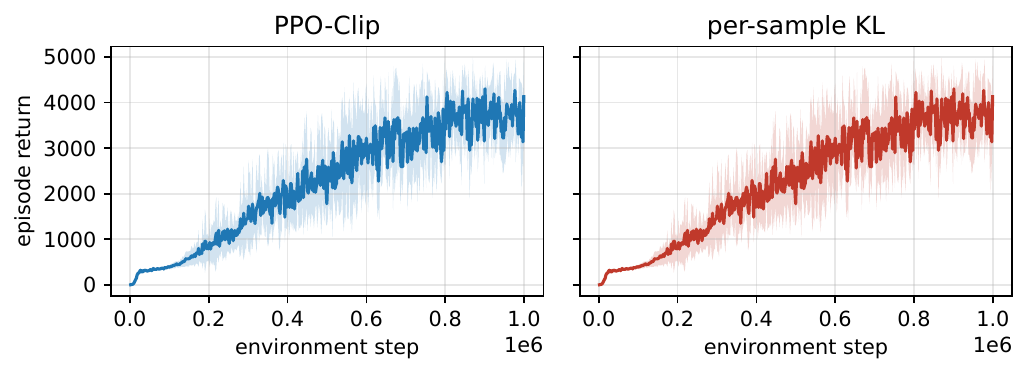}
    \caption{Walker2d-v4}
    \label{fig:id-walker2d}
  \end{subfigure}
  \hfill
  \begin{subfigure}[t]{0.49\textwidth}
    \centering
    \includegraphics[width=\textwidth]{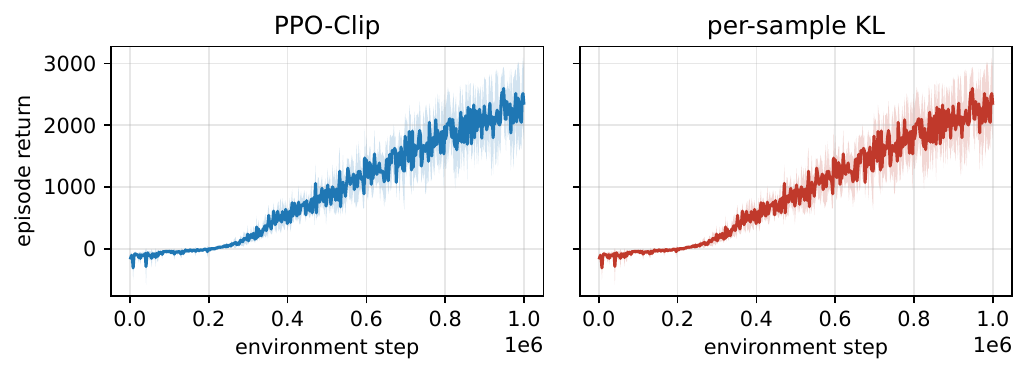}
    \caption{Ant-v4}
    \label{fig:id-ant}
  \end{subfigure}

  \begin{subfigure}[t]{0.49\textwidth}
    \centering
    \includegraphics[width=\textwidth]{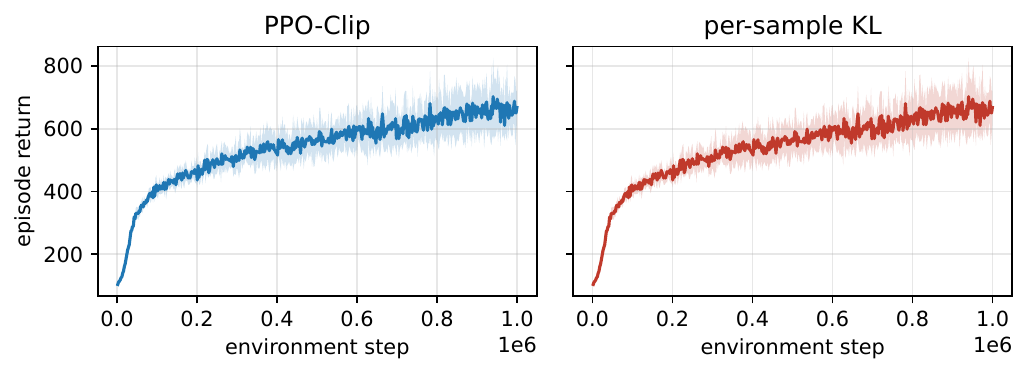}
    \caption{Humanoid-v4}
    \label{fig:id-humanoid}
  \end{subfigure}
  \caption{PPO-Clip and per-sample PPO-KL on each task,
  mean over $5$ seeds with a $\pm$ std band. The two coincide on every
  environment.}
  \label{fig:identity}
\end{figure}

\subsection{Trust-region knob sweeps}

To place each scalar-$\beta$ baseline at a defensible operating point we
sweep the trust-region knob of each variant; the main-text baselines use
the best value of each knob.

\begin{figure}[H]
  \centering
  \includegraphics[width=\textwidth]{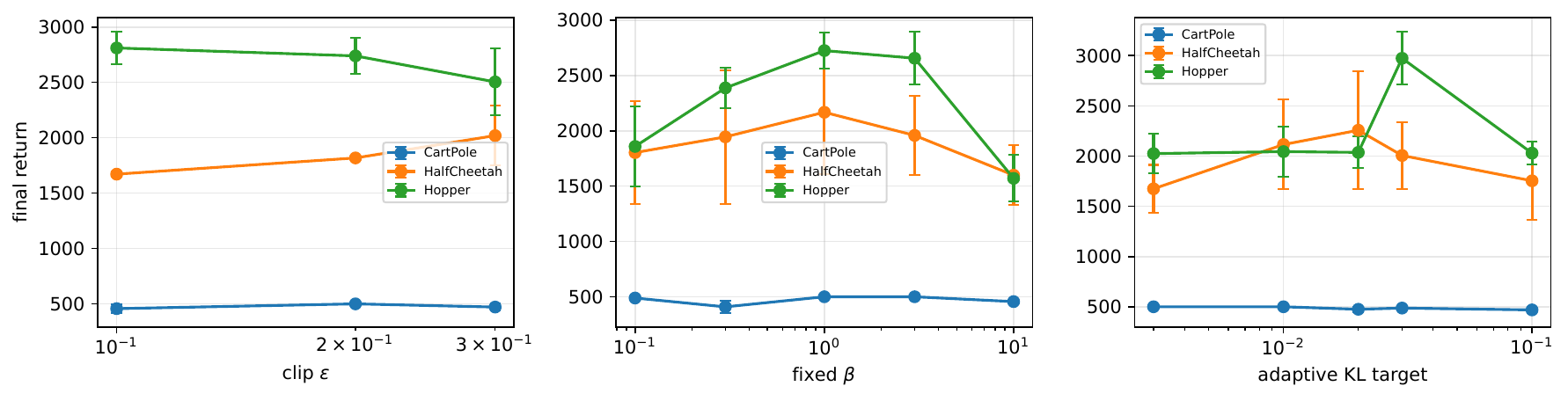}
  \caption{Final return (mean $\pm$ standard error over
  $5$ seeds) on CartPole-v1, HalfCheetah-v4, and Hopper-v4 as a function of each
  trust-region knob: clip $\epsilon$ for PPO-Clip, fixed $\beta$ for
  PPO-KL, and the KL target for adaptive PPO-KL.}
  \label{fig:knob-sweeps}
\end{figure}

\subsection{Clipping partition}

Figure~\ref{fig:partition} shows, for PPO-Clip, the fraction of each
minibatch that falls in $\mathcal{I}_{\mathrm{kill}}$ and
$\mathcal{I}_{\mathrm{pass}}$ over training. This is the empirical view of
the partition on which the identity rests: the penalty $\beta_t$ is
non-zero only on $\mathcal{I}_{\mathrm{kill}}$, and its reach grows with
task difficulty, exceeding half the batch on Humanoid-v4.

\begin{figure}[H]
  \centering
  \begin{subfigure}[t]{0.32\textwidth}
    \includegraphics[width=\textwidth]{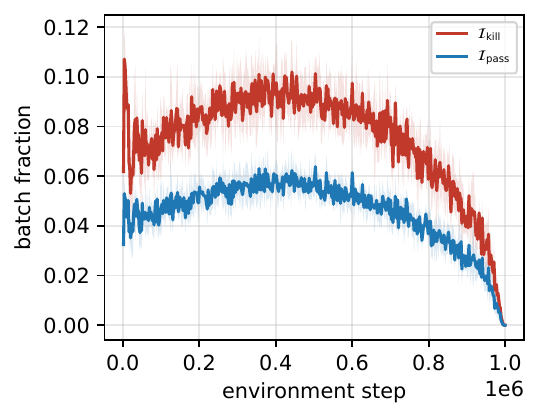}
    \caption{Hopper-v4}
  \end{subfigure}
  \hfill
  \begin{subfigure}[t]{0.32\textwidth}
    \includegraphics[width=\textwidth]{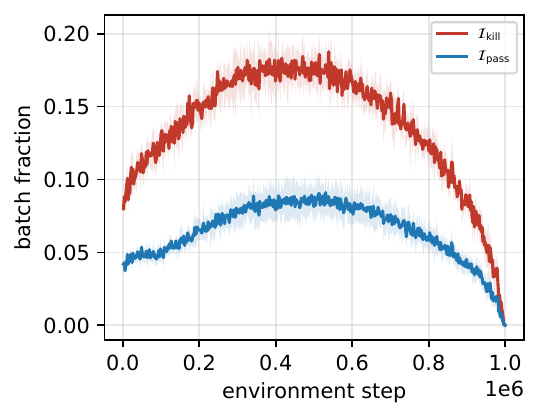}
    \caption{HalfCheetah-v4}
  \end{subfigure}
  \hfill
  \begin{subfigure}[t]{0.32\textwidth}
    \includegraphics[width=\textwidth]{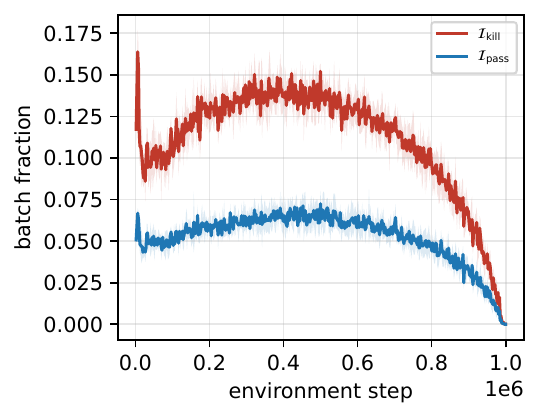}
    \caption{Walker2d-v4}
  \end{subfigure}

  \vspace{4pt}

  \begin{subfigure}[t]{0.32\textwidth}
    \includegraphics[width=\textwidth]{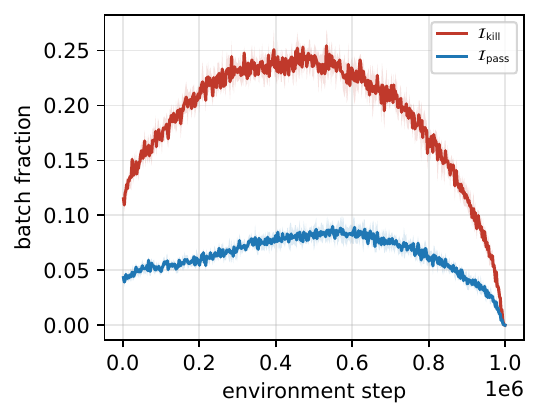}
    \caption{Ant-v4}
  \end{subfigure}
  \hspace{0.03\textwidth}
  \begin{subfigure}[t]{0.32\textwidth}
    \includegraphics[width=\textwidth]{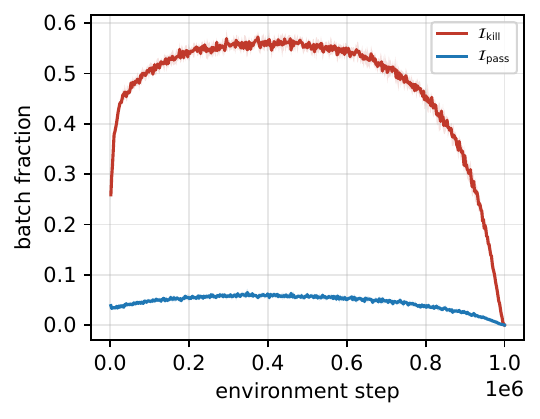}
    \caption{Humanoid-v4}
  \end{subfigure}
  \caption{Fraction of the PPO-Clip minibatch in
  $\mathcal{I}_{\mathrm{kill}}$ and $\mathcal{I}_{\mathrm{pass}}$ over
  training (mean over $5$ seeds, $\pm$ std band). The per-sample
  coefficient $\beta_t$ acts only on $\mathcal{I}_{\mathrm{kill}}$.}
  \label{fig:partition}
\end{figure}

\subsection{Per-sample coefficient}

Figure~\ref{fig:beta} plots the per-sample coefficient $\beta_t$ over
training for the per-sample variant. Its median is zero, since most
samples lie outside $\mathcal{I}_{\mathrm{kill}}$, while the tails carry
the active penalty $-w_t\hat{A}_t$ and widen on the harder tasks.

\begin{figure}[H]
  \centering
  \begin{subfigure}[t]{0.32\textwidth}
    \includegraphics[width=\textwidth]{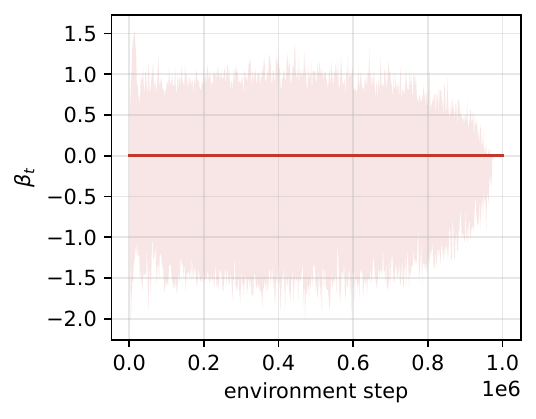}
    \caption{Hopper-v4}
  \end{subfigure}
  \hfill
  \begin{subfigure}[t]{0.32\textwidth}
    \includegraphics[width=\textwidth]{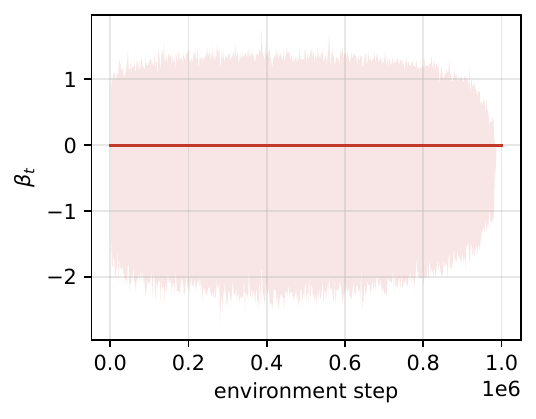}
    \caption{HalfCheetah-v4}
  \end{subfigure}
  \hfill
  \begin{subfigure}[t]{0.32\textwidth}
    \includegraphics[width=\textwidth]{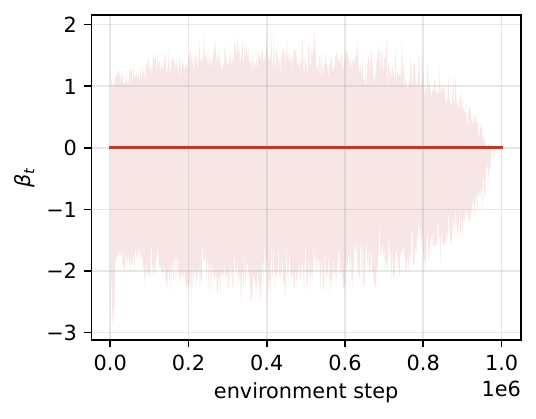}
    \caption{Walker2d-v4}
  \end{subfigure}

  \vspace{4pt}

  \begin{subfigure}[t]{0.32\textwidth}
    \includegraphics[width=\textwidth]{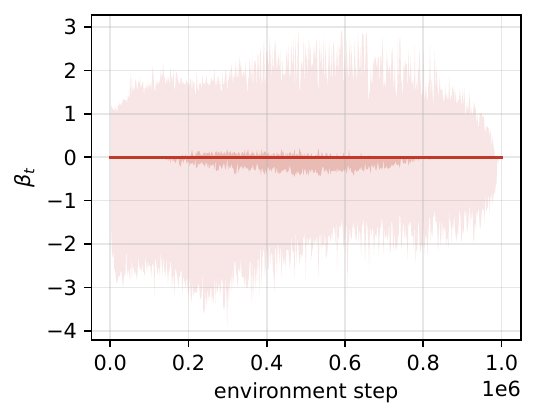}
    \caption{Ant-v4}
  \end{subfigure}
  \hspace{0.03\textwidth}
  \begin{subfigure}[t]{0.32\textwidth}
    \includegraphics[width=\textwidth]{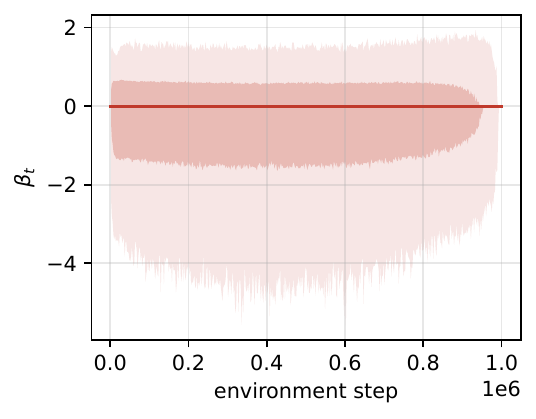}
    \caption{Humanoid-v4}
  \end{subfigure}
  \caption{Per-sample coefficient $\beta_t$ over training for the
  per-sample variant. The median is zero because $\beta_t$ vanishes on
  every sample outside the kill region; the shaded bands show how far the
  active coefficient $-w_t\hat{A}_t$ reaches on the samples in
  $\mathcal{I}_{\mathrm{kill}}$, widening on the high-dimensional tasks.}
  \label{fig:beta}
\end{figure}

\end{document}